\documentclass[wcp]{jmlr}

\usepackage{longtable}

\usepackage{booktabs}
\usepackage[small]{caption}
\usepackage[subrefformat=parens]{subcaption}
\usepackage{mathtools}
\usepackage{amsthm}
\usepackage{booktabs}
\usepackage{algorithm}
\usepackage{algorithmic}
\usepackage {diagbox}

\newtheorem{proposition}{Proposition}
\newtheorem{definition}{Definition}

\newcommand{\argmin}{\operatornamewithlimits{argmin}}
\DeclareMathOperator{\dom}{dom}
\usepackage{multirow}
\makeatletter
\let\Ginclude@graphics\@org@Ginclude@graphics 
\makeatother

\jmlrvolume{189}
\jmlryear{2022}
\jmlrworkshop{ACML 2022}

\title[Robust computation of optimal transport by  $\beta$-potential regularization]{Robust computation of optimal transport by  $\beta$-potential regularization}

  \author{\Name{Shintaro Nakamura} \Email{nakamurashintaro@g.ecc.u-tokyo.ac.jp}\\
  \addr The University of Tokyo, 5-1-5 Kashiwanoha, Kashiwa City, Chiba 277-8561
  \AND
  \Name{Han Bao} \Email{bao@i.kyoto-u.ac.jp}\\
  \addr Kyoto University, Yoshida-honmachi, Sakyo-ku, Kyoto 606-8501
  \AND
  \Name{Masashi Sugiyama} \Email{sugi@k.u-tokyo.ac.jp}\\
  \addr  RIKEN AIP center, Nihonbashi 1-chome Mitsui Building, 15th floor, 1-4-1 Nihonbashi, Chuo-ku, Tokyo 103-0027 \\
         The University of Tokyo, 5-1-5 Kashiwanoha, Kashiwa City, Chiba 277-8561
 }

\editors{Emtiyaz Khan and Mehmet Gonen}

\begin{document}

\maketitle

\begin{abstract}
Optimal transport (OT) has become a widely used tool in the machine learning field to measure the discrepancy between probability distributions. For instance, OT is a popular loss function that quantifies the discrepancy between an empirical distribution and a parametric model. Recently, an entropic penalty term and the celebrated Sinkhorn algorithm have been commonly used to approximate the original OT in a computationally efficient way. However, since the Sinkhorn algorithm runs a projection associated with the Kullback-Leibler divergence, it is often vulnerable to outliers. To overcome this problem, we propose regularizing OT with the $\beta$-potential term associated with the so-called $\beta$-divergence, which was developed in robust statistics. Our theoretical analysis reveals that the $\beta$-potential can prevent the mass from being transported to outliers. We experimentally demonstrate that the transport matrix computed with our algorithm helps estimate a probability distribution robustly even in the presence of outliers. In addition, our proposed method can successfully detect outliers from a contaminated dataset.
\end{abstract}

\section{Introduction}
Many machine learning problems such as density estimation and generative modeling are often formulated by a discrepancy between probability distributions \citep{leastsquare,GAN}. As a common choice, the Kullback-Leibler (KL) divergence \citep{KLdivergence} has been widely used since minimizing the KL-divergence of an empirical distribution from a parametric model corresponds to maximum likelihood estimation. However, the KL-divergence suffers from some problems. For instance, the KL-divergence of $p$ from $q$ is not well-defined when the support of $p$ is not completely included in the support of $q$. Moreover, the KL-divergence does not satisfy the axioms of metrics in a probability space. On the other hand, \emph{optimal transport} (OT) \citep{OTbook} does not suffer from these problems. OT does not require any conditions on the support of probability distributions and thus is expected to be more stable than the KL-divergence. Therefore, the divergence estimator is less prone to diverge to infinity. In addition, OT between two distributions is a metric in a probability space and therefore defines a proper distance between histograms and probability measures \citep{OTformulation}. Owing to these nice properties, OT has been celebrated with many applications such as image processing \citep{Rabin} and color modifications \citep{Solomon}. \par
However, the ordinary OT suffers from heavy computation. To cope with this problem, one of the common approaches is to regularize the ordinary OT problem with an entropic penalty term (Boltzmann--Shannon entropy \citep{ROTandRMD}) and use the Sinkhorn algorithm \citep{Sinkhorn1967} to approximate OT \citep{Sinkhorn}. The entropic penalty makes the objective strictly convex, ensuring the existence of the unique global optimal solution, and the Sinkhorn algorithm projects this global optimal solution onto a set of couplings in terms of the KL-divergence, a divergence associated with the Boltzmann--Shannon entropy \citep{ROTandRMD}.
Unfortunately, the KL projection in statistical estimation is often not robust in the presence of outliers \citep{Basu}. In our pilot study, we experimentally confirmed that the Sinkhorn algorithm is easily affected by outliers (Figure \ref{toyexperiment}). As can be seen in Table \ref{table_toyexperiment}, the output value of the Sinkhorn algorithm  drastically increases even when only a small number of outliers are included in the dataset.
\begin{figure}[t]
    \begin{tabular}{cc}
      \begin{minipage}[t]{0.45\linewidth}
        \centering
        \includegraphics[width = 4.0cm]{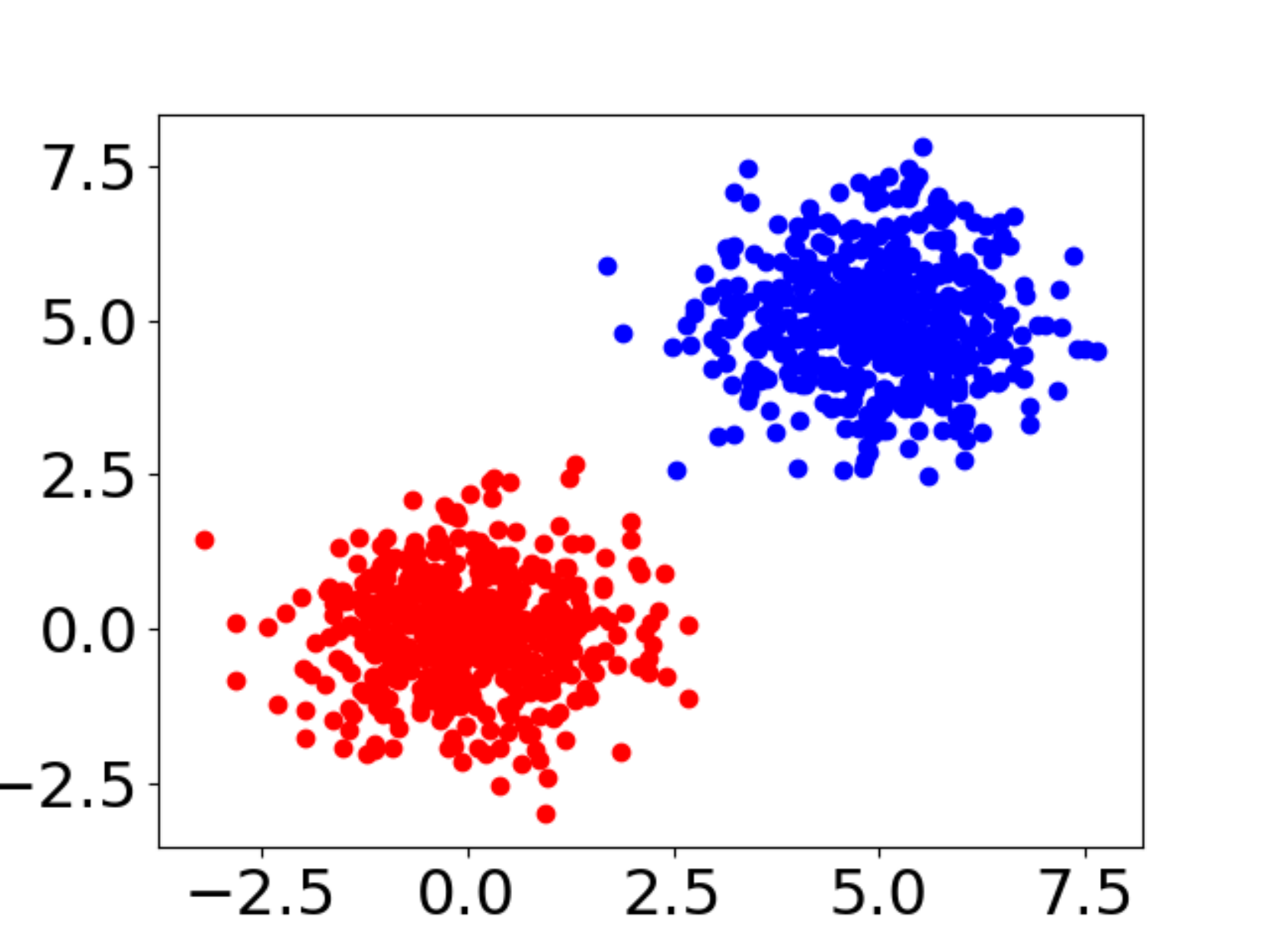}
        \subcaption{Sets of samples without outliers}
        \label{figure(a)}
      \end{minipage} &
      \begin{minipage}[t]{0.45\linewidth}
        \centering
        \includegraphics[width = 4.0cm]{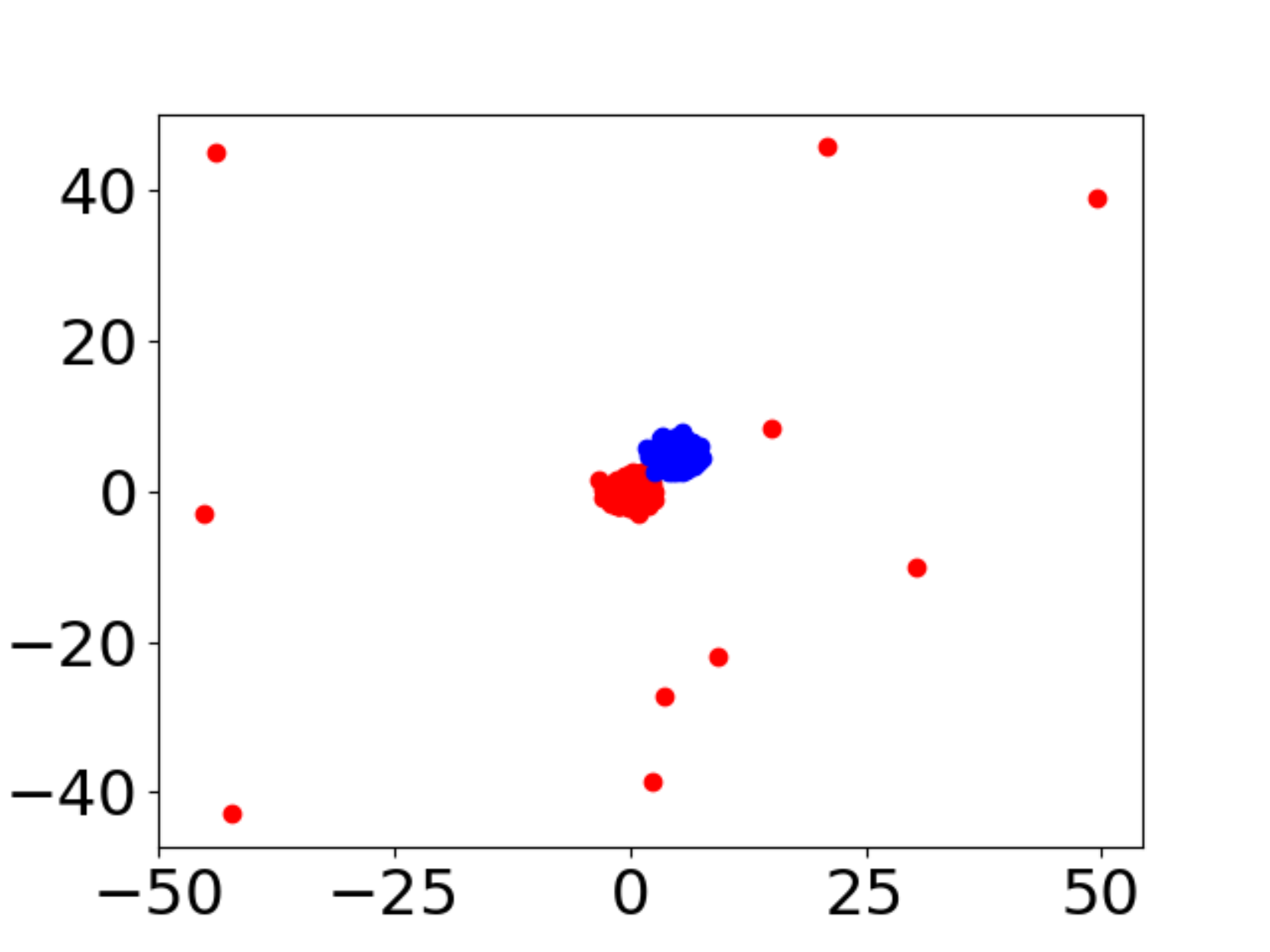}
        \subcaption{Sets of samples with outliers}
        \label{figure(b)}
      \end{minipage}
    \end{tabular}
      \caption{(a) 500 samples (red) are drawn from $\mathcal{N}([0, 0]^{\top}, I)$ and  500 samples (blue) are from $\mathcal{N}([5, 5]^{\top}, I)$. $I$ is the two-dimensional identity matrix. (b) 10 samples from two-dimensional uniform distribution $\mathrm{U}\{(x, y)|-50\le x, y \le 50 \}$ are added to the red samples.}
  \label{toyexperiment}
 \end{figure}

\begin{table}[t]
    \caption{The ouput value of the Sinkhorn algorithm and our algorithm. The exact OT value is 50.13 for the sets of samples in Figure (\ref{figure(a)}).}
    \centering
  \begin{tabular}{ccc} 
    \toprule  & Figure~\ref{figure(a)} & Figure~\ref{figure(b)}  \\ 
    \midrule
    The Sinkhorn algorithm & 50.74 &  92.19 \\
    Our algorithm & 50.10 & 50.00 \\ 
    \bottomrule
  \end{tabular} 
    
  \label{table_toyexperiment}
\end{table}

The high sensitivity of the Sinkhorn algorithm may lead to undesired solutions in probabilistic modeling when we deal with noisy and adversarial datasets \citep{adversarial_attack}. Several existing works have tackled this challenge. \citeauthor{Staerman}(\citeyear{Staerman}) proposed a median-of-means estimator of the $1$-Wasserstein dual to suppress outlier sensitivity. However, the obtained solution is hard to be interpreted as an approximation to OT because the corresponding primal problem is unclear. On the other hand, the following works robustly approximate OT by sending only a small probability mass to outliers, allowing some violation of the coupling constraint: \citeauthor{Balaji}(\citeyear{Balaji}) used unbalanced OT \citep{Chizat2017} with the $\chi^{2}$-divergence as their $f$-divergence penalty on marginal violation to compute OT robustly. This formulation requires access to the outlier proportion, which is usually not available. Moreover, in Section \ref{outlierdetection_experiment}, we show their method relying on the optimization package CVXPY \citep{CVXPY} does not scale to large samples. \citeauthor{Mukherjee}(\citeyear{Mukherjee}) mainly focused on outlier detection by truncating the distance matrix in OT. As a downside, one needs to set an appropriate threshold to use their method, which is hardly known in advance. Hence, we still lack a robust OT formulation independent of sensitive hyperparameters, with easily-accessible primal transport matrices.\par
In this work, we propose to mitigate the outlier sensitivity of the Sinkhorn algorithm by regularizing OT with the $\beta$-potential term instead of the Boltzmann--Shannon entropy. This formulation can be regarded as a projection based on the $\beta$-divergence~\citep{Basu,Futami}.
With some computational tricks, our algorithm is guaranteed not to move any probability mass to outliers (Figure~\ref{P_beta_figures}). It also suggests that our algorithm computes an approximate OT between the inliers. The approximate OT computed by our method was 50.10 and 50.00 in the settings of Figures~\ref{figure(a)} and~\ref{figure(b)}, respectively (Table \ref{table_toyexperiment}), meaning that our method is less prone to be affected by outliers. Through numerical experiments, we demonstrate that our proposed method can measure a distance between datasets more robustly than the Sinkhorn algorithm. As a practical application, we show our proposed method can be applied to an outlier detection task.
\begin{figure}[t]
    \centering
    \includegraphics[width = 10cm]{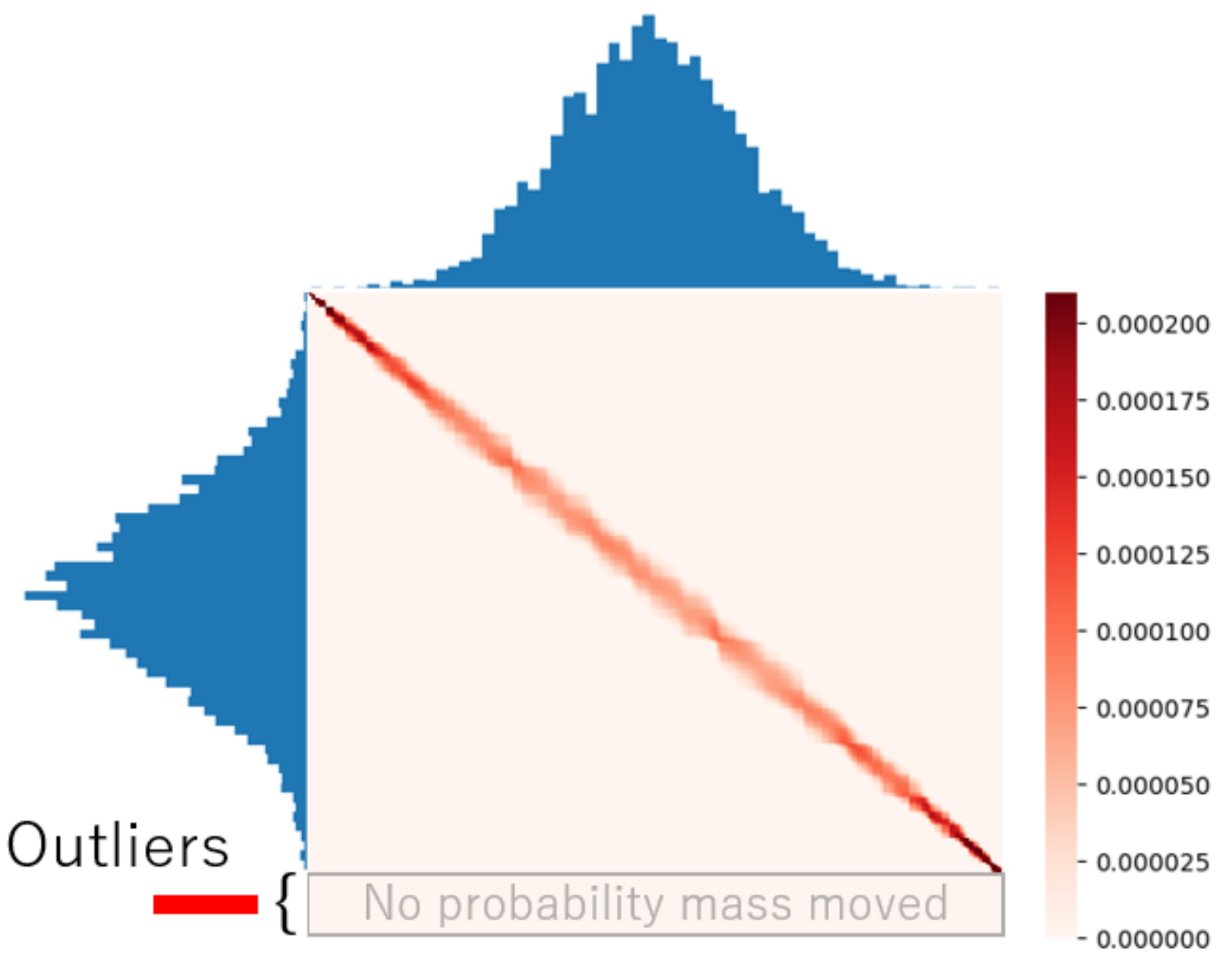}
    \caption{The heatmap of the transport matrix computed with our algorithm. The horizontal histogram is a set of 500 samples from a 1-dimensional standard normal distribution. The vertical histogram is a set of 495 samples from a 1-dimensional standard normal distribution added with 5 outliers with a value of 70. As can be seen from the heatmap, no mass was transported from outliers included in the source histogram.}
    \label{P_beta_figures}
\end{figure}

\section{Background}
In this section, we first show the formulation of ordinary discrete OT. Subsequently, we review the Bregman divergence. Finally, we introduce the convex regularized discrete OT (CROT) formulation and the alternate Bregman projection to obtain the solution to the CROT.
\subsection{Optimal Transport (OT)}
We introduce OT in a discrete setting. In this case, OT can be regarded as the cheapest plan to deliver items from $m$ suppliers to $n$ consumers, where each supplier and consumer has supply $\frac{1}{m}$ and demand $\frac{1}{n}$, respectively. In this work, we mainly focus on measuring the transportation cost between two probability distributions. Suppose we have two sets of independent samples $\{ \boldsymbol{x}_i \}_{i = 1}^{m}$ and $\{ \boldsymbol{y}_j\}_{j = 1}^{n}$ drawn from two distributions $P_x$ and $P_y$, respectively. We write the corresponding empirical measures by $\hat{P}_{x} \coloneqq \frac{1}{m} \sum_{i = 1}^{m} \boldsymbol{x}_i \delta_{\boldsymbol{x}_i}$ and $\hat{P}_{y} \coloneqq \frac{1}{n} \sum_{i = 1}^{n} \boldsymbol{y}_i \delta_{\boldsymbol{y}_i}$, where $\delta_{\boldsymbol{x}}$ is the delta function at position $\boldsymbol{x}$. Let $\boldsymbol{\gamma} \in \mathbb{R}_{+}^{m \times n}$ be the distance matrix, where $\gamma_{ij}$ denotes the distance between $\boldsymbol{x}_i$ and $\boldsymbol{y}_j$. Transport matrices are confined to 
\begin{align}\label{coupling_constraint}
    \left\{ \boldsymbol{\Pi} \in \mathbb{R}^{m \times n}_{+} \,\middle|\, \boldsymbol{\Pi} \boldsymbol{1}_n = \frac{\boldsymbol{1}_m}{m}, \boldsymbol{\Pi}^{\top} \boldsymbol{1}_m = \frac{\boldsymbol{1}_n}{n} \right\} \eqqcolon \mathcal{G}\left(\frac{\boldsymbol{1}_m}{m}, \frac{\boldsymbol{1}_n}{n}\right),
\end{align}
where $\mathbb{R}^{m \times n}_{+}$ is the set of non-negative reals. We call $\mathcal{G}(\frac{\boldsymbol{1}_m}{m}, \frac{\boldsymbol{1}_n}{n})$ the \emph{coupling constraint}.
In order to keep the notation concise, the Frobenius inner product between two matrices $\boldsymbol{\pi}, \boldsymbol{\gamma} \in \mathbb{R}_{+}^{m \times n}$ is denoted by
    \begin{equation}
        \langle \boldsymbol{\pi}, \boldsymbol{\gamma} \rangle \coloneqq \sum_{i, j}\pi_{ij}\gamma_{ij}.
    \end{equation}
Then, OT between the two empirical distributions $\hat{P}_{x}$ and $\hat{P}_{y}$ is defined as follows \citep{OTformulation}:
    \begin{eqnarray}\label{DiscreteOT}
        \mathrm{OT}(\hat{P}_{x} \| \hat{P}_{y}) &\coloneqq& \min_{\boldsymbol{\pi}\in \mathcal{G}(\frac{\boldsymbol{1}_m}{m}, \frac{\boldsymbol{1}_n}{n})} \langle\boldsymbol{\pi}, \boldsymbol{\gamma} \rangle.
    \end{eqnarray}

\subsection{Bregman divergence}
Let $\mathcal{E}$ be a Euclidean space with inner product $\langle \cdot, \cdot \rangle$ and induced norm $\|\cdot\|$.
Let $\phi:\mathcal{E}\to\mathbb{R}$ be a strictly convex function on $\mathcal{E}$ that is differentiable on $\mathrm{int}(\dom\phi) \neq \emptyset$. The \emph{Bregman divergence} generated by $\phi$ is defined as follows:
\begin{equation}
    B_{\phi}(\boldsymbol{x} \| \boldsymbol{y}) := \phi(\boldsymbol{x}) - \phi(\boldsymbol{y}) - \langle \boldsymbol{x} - \boldsymbol{y},  \nabla \phi(\boldsymbol{y})\rangle, 
\end{equation}
for all $\boldsymbol{x} \in \dom\phi$ and $\boldsymbol{y} \in \dom\phi$. In this paper, for the sake of simplicity, we consider the so-called \emph{separable} Bregman divergences \citep{ROTandRMD} over the set of transport matrices $\mathcal{G}\left( \frac{\boldsymbol{1}_m}{m}, \frac{\boldsymbol{1}_n}{n} \right)$, which can be decomposed as the element-wise summation:
\begin{eqnarray}
    B_{\phi}(\boldsymbol{\pi} \| \boldsymbol{\xi}) &=& \sum_{i = 1}^{m} \sum_{j = 1}^{n} B_{\phi}(\pi_{ij}\|\xi_{ij}), \\
    \phi(\boldsymbol{\pi}) &=& \sum_{i = 1}^{m} \sum_{j = 1}^{n} \phi(\pi_{ij}),
\end{eqnarray}
where we used $\phi: \mathbb{R} \to \mathbb{R}$ to denote the generator function same across all elements, with a slight abuse of notation.
Suppose now that $\phi$ is of the Legendre type \citep{Heinz}, and let $\mathcal{C} \subseteq \mathcal{E}$ be a closed convex set such that $\mathcal{C}\cap\mathrm{int}( \dom \phi ) \neq \emptyset$. Then, for any point $\boldsymbol{y} \in \mathrm{int}(\dom\phi)$, the following problem,
\begin{equation}
    T_{\mathcal{C}}(\boldsymbol{y}) = \argmin_{\boldsymbol{x} \in \mathcal{C}} B_{\phi}(\boldsymbol{x} \| \boldsymbol{y}),
\end{equation}
has a unique solution. $T_{\mathcal{C}}(\boldsymbol{y})$ is called the \emph{Bregman} projection of $\boldsymbol{y}$ onto $\mathcal{C}$ \citep{ROTandRMD}.

\subsection{Formulation of CROT}
\begin{table}[t]
    \centering
  \begin{tabular}{ccc} 
    \toprule Regularization term & dom $\phi$ & dom $\psi$  \\ 
    \midrule
    $\beta$-potential ($\beta > 1$) & $\mathbb{R}_{+}$ &  $(\frac{1}{1 - \beta}, \infty)$ \\
    Boltzman-Shannon entropy& $\mathbb{R}_{+}$ & $\mathbb{R}$ \\ 
    \bottomrule
  \end{tabular} 
    \caption{Domains of each regularizer and its Fenchel conjugate.}
  \label{domain_of_regularizer}
\end{table}


Here, we give the formulation of the CROT and show that obtaining the optimal solution of the CROT corresponds to minimizing the Bregman divergence between two matrices. \par
The CROT is formulated as a regularized version of (\ref{DiscreteOT}) by $\phi$ as follows:
\begin{equation}\label{CROT}
    L_{\phi} (\boldsymbol{\pi}) \coloneqq \min_{\boldsymbol{\pi}\in \mathcal{G}(\frac{\boldsymbol{1}_m}{m}, \frac{\boldsymbol{1}_n}{n})} \langle \boldsymbol{\pi}, \boldsymbol{\gamma} \rangle + \lambda \phi(\boldsymbol{\pi}),
\end{equation}
where $\lambda > 0$ is a regularization parameter. Subsequently, we often work on the dual variable of $\boldsymbol{\pi}$. The dual variable $\boldsymbol{\theta}$ satisfies the following conditions:\footnote{This mapping based on gradients (not subgradients) is legitimate only when $\phi$ is of the Legendre type.}
\begin{eqnarray}
    \boldsymbol{\pi} &=& \nabla\psi(\boldsymbol{\theta}), \\
    \boldsymbol{\theta} & = & \nabla\phi(\boldsymbol{\pi}),
\end{eqnarray}
where $\psi$ is the Fenchel conjugate of $\phi$ \citep{ROTandRMD}.
The optimal solution of (\ref{CROT}) can be understood via the Bregman projection.
Let us consider the unconstrained version of (\ref{CROT}):
\begin{eqnarray}\label{non_constraint_CROT}
    \min_{\boldsymbol{\pi} \in \mathbb{R}^{m \times n}} \langle \boldsymbol{\pi}, \boldsymbol{\gamma} \rangle + \lambda \phi(\boldsymbol{\pi}).
\end{eqnarray}
Since $\langle \boldsymbol{\pi}, \boldsymbol{\gamma} \rangle$ is linear and $\phi$ is strictly convex with respect to $\boldsymbol{\pi}$, there is a unique optimal solution $\boldsymbol{\xi}$ for (\ref{non_constraint_CROT}):
\begin{equation}\label{xi}
    \boldsymbol{\xi} = \nabla\psi(-\boldsymbol{\gamma} / \lambda),
\end{equation}
which can be obtained by solving the first-order optimality condition of (\ref{non_constraint_CROT}) with the dual relationship $(\nabla \phi)^{-1} = \nabla\psi$:
\begin{equation}
    \boldsymbol{\gamma} + \lambda \nabla\phi(\boldsymbol{\xi}) = 0.
\end{equation}
Then, 
\begin{align}
    \boldsymbol{\pi}_{\lambda}^{*} &\coloneqq \argmin_{\boldsymbol{\pi} \in \mathcal{G}(\frac{\boldsymbol{1}_m}{m}, \frac{\boldsymbol{1}_n}{n})} L_{\phi}(\boldsymbol{\pi})\\
    &= \argmin_{\boldsymbol{\pi} \in \mathcal{G}(\frac{\boldsymbol{1}_m}{m}, \frac{\boldsymbol{1}_n}{n})}  B_{\phi}(\boldsymbol{\pi}\|\boldsymbol{\xi}),
\end{align}
where the last equality is due to the following equation:
\begin{equation}
    \langle \boldsymbol{\pi}, \boldsymbol{\gamma} \rangle + \lambda \phi(\boldsymbol{\pi}) - \lambda\phi(\boldsymbol{\xi}) - \langle \boldsymbol{\xi}, \boldsymbol{\gamma} \rangle = \lambda B_{\phi}(\boldsymbol{\pi} \| \boldsymbol{\xi}).
\end{equation}
Therefore, the solution of (\ref{CROT}) can be interpreted as the Bregman projection of the unconstrained solution $\boldsymbol{\xi}$ onto $\mathcal{G}(\frac{\boldsymbol{1}_m}{m}, \frac{\boldsymbol{1}_n}{n})$.
The Sinkhorn algorithm can be used to obtain a solution to OT regularized with the negative of Boltzmann--Shannon entropy
\begin{math}
    \phi(\pi) = \pi \log \pi - \pi + 1
\end{math}
(Table \ref{domain_of_regularizer}), and runs a projection associated with the KL-divergence where $B_{\phi}(\pi\|\xi) = \pi \log\frac{\pi}{\xi} - \pi + \xi$.
\subsection{Alternate Bregman projection}

Here, we demonstrate how the Bregman projection onto $\mathcal{G}(\frac{\boldsymbol{1}_m}{m}, \frac{\boldsymbol{1}_n}{n})$ is executed based on \citeauthor{ROTandRMD}[\citeyear{ROTandRMD}]. \par
Let $\mathcal{C}_0, \mathcal{C}_1, \mathcal{C}_2$ be the following convex sets:
\begin{eqnarray}
    \mathcal{C}_0 &=& \mathbb{R}_{+}^{m \times n}, \\
    \mathcal{C}_1 &=& \left\{ \boldsymbol{\pi} \in \mathbb{R}^{m \times n} \,\middle|\,  \boldsymbol{\pi}\boldsymbol{1}_n = \textstyle{\frac{\boldsymbol{1}_m}{m}} \right\}, \\
    \mathcal{C}_2 &=& \left\{ \boldsymbol{\pi} \in \mathbb{R}^{m \times n} \,\middle|\,  \boldsymbol{\pi}^{\top} \boldsymbol{1}_m = \textstyle{\frac{\boldsymbol{1}_n}{n}} \right\}. 
\end{eqnarray}
Then, $\mathcal{G}(\frac{\boldsymbol{1}_m}{m}, \frac{\boldsymbol{1}_n}{n})$ can be written as follows:
\begin{equation}
    \mathcal{G}(\textstyle{\frac{\boldsymbol{1}_m}{m}}, \textstyle{\frac{\boldsymbol{1}_n}{n}}) = \mathcal{C}_0 \cap \mathcal{C}_1 \cap \mathcal{C}_2.
\end{equation}
We can get the Bregman projection onto $\mathcal{G}(\frac{\boldsymbol{1}_m}{m}, \frac{\boldsymbol{1}_n}{n})$ by alternately performing projections onto $\mathcal{C}_0$, $\mathcal{C}_1$, and $\mathcal{C}_2$. \par

Next, let us consider the projection of a given matrix $\overline{\boldsymbol{\pi}} \in \mathrm{int}(\dom \phi)$ onto $\mathcal{C}_0$, $\mathcal{C}_1$, and $\mathcal{C}_2$. The corresponding projection onto each set is denoted by $\boldsymbol{\pi}_{0}^{*}$, $\boldsymbol{\pi}_{1}^{*}$, and $\boldsymbol{\pi}_{2}^{*}$, respectively. Subsequently, we show how to obtain them \citep{ROTandRMD} (see Section A in the supplementary file for details).
\subsubsection{Projection onto $\mathcal{C}_0$}
When considering the separable Bregman divergence, the projection onto $\mathcal{C}_{0}$ can be performed with a closed-form expression in terms of primal parameters:
\begin{equation}\label{non-negativity_constraint}
    \pi^{*}_{0, ij} = \max \{ 0, \overline{\pi}_{ij} \},
\end{equation}
where, $\pi^{*}_{0, ij}$ is the $(i, j)$-element of matrix $\boldsymbol{\pi}_{0}^{*}$. 
Since $\phi'$ is increasing, this is equivalently expressed in terms of the dual parameters of $\boldsymbol{\pi}_0^*$, $\boldsymbol{\theta}_0^*$, as
\begin{equation} \label{projectionC0}
    \theta_{0, ij}^{*} = \max \{ \phi'(0), \overline{\theta}_{ij} \}.
\end{equation}
Here, the dual coordinate of the input matrix $\overline{\boldsymbol{\pi}}$ is denoted by $\overline{\boldsymbol{\theta}} = \nabla \psi (\overline{\boldsymbol{\pi}})$.

\subsubsection{Projections onto $\mathcal{C}_1$ and $\mathcal{C}_2$}
Next, we consider the Bregman projection onto $\mathcal{C}_{1}$. The projection onto $\mathcal{C}_2$ can be executed in the same way and thus omitted here. The Lagrangian associated to the Bregman projection $\boldsymbol{\pi}_{1}^{*}$ of a given matrix $\overline{\boldsymbol{\pi}} \in \mathrm{int}(\dom\phi)$ onto $\mathcal{C}_1$ is given as follows:
\begin{align}
    \mathcal{L}_1 (\boldsymbol{\pi}, \boldsymbol{\mu})  =   \phi(\boldsymbol{\pi}) - \langle \boldsymbol{\pi},&  \nabla\phi(\overline{\boldsymbol{\pi}}) \rangle \nonumber + \boldsymbol{\mu}^{\top}(\boldsymbol{\pi} \boldsymbol{1}_n - \textstyle{\frac{\boldsymbol{1}_m}{m}}),
\end{align}
where $\boldsymbol{\mu} \in \mathbb{R}^{m}$ are Lagrange multipliers.
Their gradients are given on $\mathrm{int}(\dom\phi)$ by
\begin{align}
    \nabla_{\boldsymbol{\pi}}\mathcal{L}_{1}(\boldsymbol{\pi}, \boldsymbol{\mu})  =  \nabla_{\boldsymbol{\pi}}\phi(\boldsymbol{\pi}) - \nabla_{\boldsymbol{\pi}}\phi(\overline{\boldsymbol{\pi}}) + \boldsymbol{\mu}{\boldsymbol{1}_{n}}^{\top},
\end{align}
and by noting $(\nabla\phi)^{-1} = \nabla\psi$, $\nabla_{\boldsymbol{\pi}}\mathcal{L}_1(\boldsymbol{\pi}_1, \boldsymbol{\mu}) = \boldsymbol{0}_{m\times n}$ if and only if \citep{ROTandRMD},
\begin{align}\label{vanish_condition}
    \boldsymbol{\pi}_{1}^{*} =  \nabla\psi( \nabla\phi(\overline{\boldsymbol{\pi}}) - \boldsymbol{\mu}{\boldsymbol{1}_{n}}^{\top}).
\end{align}
By multiplying $\boldsymbol{1}_n$ on the both sides of (\ref{vanish_condition}), the following equation system is obtained:
\begin{align}
    \nabla\psi (\nabla\phi(\overline{\boldsymbol{\pi}}) - \boldsymbol{\mu}{\boldsymbol{1}_n}^{\top}) \boldsymbol{1}_n  =  \textstyle{\frac{\boldsymbol{1}_m}{m}}.
\end{align}
Due to the separability, the projection onto $\mathcal{C}_{1}$  can be divided into $m$ subproblems in each coordinate of the dual variable as follows:
\begin{align}\label{m_subproblem}
    \sum_{j = 1}^{n} \psi'(\overline{\theta}_{ij} - \mu_{i})  =  \frac{1}{m}.
\end{align}
To solve equation (\ref{m_subproblem}) with respect to $\mu_{i}$, we use the Newton--Raphson method \citep{NewtonRaphson}. \par
\begin{algorithm}[t]
\caption{Non-negative alternate scaling algorithm for $\beta$-divergence when $\beta > 1$}
\label{ouralgorithm}
\begin{algorithmic}[1]
\STATE $\tilde{\boldsymbol{\theta}} \leftarrow - \boldsymbol{\gamma} / \lambda$
\STATE\label{line2} $\boldsymbol{\theta^{*}} \leftarrow  \max \{ \nabla\phi(\boldsymbol{0}_{m \times n}), \tilde{\boldsymbol{\theta}} \}$ 
\FOR{$t$ = 1, 2, \ldots, $T$} \label{line3}
  \STATE\label{line4} $\boldsymbol{\tau} = \frac{\nabla\psi(\boldsymbol{\theta^{*}})\boldsymbol{1}_n - \frac{\boldsymbol{1}_m}{\mathit{m}}}{\nabla^{2}\psi(\boldsymbol{\theta}^{*}) \boldsymbol{1}_n}$
  \STATE\label{line5} $\boldsymbol{\tau} \leftarrow \max (\boldsymbol{\tau},\ \hat{\boldsymbol{\theta}}^{*} - \nabla\phi(\frac{\boldsymbol{1}_m}{m}))$ 
  \STATE\label{line6} $\boldsymbol{\tilde{\theta}} \leftarrow \boldsymbol{\tilde{\theta}} - \boldsymbol{\tau} {\boldsymbol{1}_n}^{\top}$
  \STATE\label{line7} $\boldsymbol{\theta^{*}} \leftarrow \max \{\nabla\phi(\bf{0}), \boldsymbol{\tilde{\theta}}\}$
  \STATE\label{line8} $\boldsymbol{\sigma} = \frac{{\boldsymbol{1}_m}^{\top} \nabla \psi(\boldsymbol{\theta^{*}}) - (\frac{\boldsymbol{1}_n}{\mathit{n}})^{\top}}{{\boldsymbol{1}_m}^{\top} \nabla^{2}\psi(\boldsymbol{\theta}^{*})}$ 
  \STATE\label{line9} $\boldsymbol{\sigma} \leftarrow \max (\boldsymbol{\sigma},\hat{\boldsymbol{\theta}}^{*} - \nabla\phi(\frac{\boldsymbol{1}_n}{\mathit{n}}))$
  \STATE\label{line10} $\boldsymbol{\tilde{\theta}} \leftarrow \boldsymbol{\tilde{\theta}} - \boldsymbol{1}_m \boldsymbol{\sigma} $ 
  \STATE\label{line11} $\boldsymbol{\theta^{*}} \leftarrow \max \{\nabla\phi(\boldsymbol{0}_{m \times n}), \boldsymbol{\tilde{\theta}}\}$
\ENDFOR\label{line12}
\STATE $\boldsymbol{\pi^{*} \leftarrow \nabla \psi(\theta^{*})}$ 
\end{algorithmic}
\end{algorithm}

\section{Outlier-robust CROT}
In this section, we first formalize a model of outliers. To make the CROT robust against outliers under the model, we propose the CROT with the $\beta$-potential ($\beta > 1$) and introduce how to compute the CROT with the $\beta$-potential. Finally, we show its theoretical properties.
\subsection{Definition of outliers}
In this paper, outliers are formally defined as follows. Suppose we have two datasets $\{ \boldsymbol{x}_i\}_{i = 1}^{m}$ and $\{\boldsymbol{y}_j\}_{j = 1}^{n}$. We assume $\{\boldsymbol{x}_i\}_{i = 1}^{m}$ are samples from a clean distribution, while $\{\boldsymbol{y}_j\}_{j = 1}^{n}$ are samples that are contaminated by outliers. Let $\boldsymbol{\gamma}$ be the distance matrix.

\begin{definition} \label{outlier_definition}
For $z>0$, the indices of  outliers $J$ are defined as follows:
\begin{equation}
    \forall j \in J, \ \forall i \in \{1, \ldots, m\}, \gamma_{ij} \geq z.
\end{equation}
\end{definition}

This means that any point in $\{\boldsymbol{y}_j\}_{j = 1}^{n}$ that is more than or equal to $z$ away from any point in $\{\boldsymbol{x}_i\}_{i = 1}^{m}$ is considered as outliers.\par

\subsection{$\beta$-potential regularization}
We use the $\beta$-potential
\begin{equation}
    \phi(\pi) = \frac{1}{\beta(\beta - 1)} (\pi^{\beta} - \beta \pi + \beta - 1),
\end{equation}
associated with the  $\beta$-divergence,
\begin{equation}
    B_{\phi}(\pi \| \xi) = \frac{1}{\beta(\beta - 1)} (\pi^{\beta} + (\beta - 1) \xi^{\beta} - \beta \pi \xi^{\beta - 1}),
\end{equation}
to robustify the CROT, where $\beta > 1$.
The domains of primal $\phi$ and its Fenchel conjugate $\psi$ are shown in Table \ref{domain_of_regularizer}. \par
Our proposed algorithm is shown in Algorithm \ref{ouralgorithm}. The dual coordinate of the unconstrained CROT solution is denoted by $\boldsymbol{\tilde{\theta}} = \nabla \phi (\boldsymbol{\xi})$. We execute the projections in the cyclic order of $\mathcal{C}_0 \rightarrow \mathcal{C}_1 \rightarrow \mathcal{C}_0 \rightarrow \mathcal{C}_2\rightarrow \mathcal{C}_0 \rightarrow \mathcal{C}_1 \rightarrow \mathcal{C}_0 \rightarrow \mathcal{C}_2 \rightarrow \cdots$. \par

Lines \ref{line2}, \ref{line7}, and \ref{line11} in Algorithm \ref{ouralgorithm} enforce the dual constraint $\theta_{ij}^* \ge \frac{1}{1-\beta}$ corresponding to $\dom\psi = (\frac{1}{1-\beta}, \infty)$  (Table \ref{domain_of_regularizer}).
Lines \ref{line4}--\ref{line6} correspond to the projection onto $\mathcal{C}_1$ implemented on the dual coordinate. Since the dual variable must satisfy $\theta_{ij}^* \ge \frac{1}{1-\beta}$ due to $\dom\psi = (\frac{1}{1-\beta}, \infty)$, we update the dual variable only once in the Newton--Raphson method (line \ref{line4}) since $\theta_{ij}^* \ge \frac{1}{1-\beta}$ is no longer guaranteed after the first update. Similarily, the projection onto $\mathcal{C}_2$ is shown in lines \ref{line8}--\ref{line10}. \par
The procedure in line~\ref{line5} is based on Section 4.6 in \citeauthor{ROTandRMD}(\citeyear{ROTandRMD}) accelerating the convergence of Algorithm~\ref{ouralgorithm} by truncating the optimization variable $\boldsymbol{\tau}$, which we describe subsequently. Recall that, for any $i$, we have the following condition, 
\begin{equation}
 \label{condition_of_pi}
 \forall j, \ 0\le \pi_{1, ij}^* \le \textstyle{\frac{1}{m}},
\end{equation}
implicitly from the coupling constraint (\ref{coupling_constraint}). Since naively updating Newton--Raphson method can ``overshoot'', we truncate $\boldsymbol{\tau}$ so that (\ref{condition_of_pi}) is satisfied after each update. Below, we show this condition is satisfied mathematically. Let $\hat{\boldsymbol{\theta}}^{*}$ be the $m$-dimensional vector whose $i$th element is the largest value in the $i$th row of $\boldsymbol{\theta}^{*}$ defined as follows:
\begin{eqnarray}
    \hat{\theta}^{*}_{i} &:=& \mathrm{max}\{ \theta^{*}_{ij} \}_{1 \le j \le n}.
\end{eqnarray}
Since $\phi$ is convex, 
\begin{eqnarray} 
    &0\le \pi_{1, ij}^* \le \frac{1}{m}& \nonumber\\
    \iff &  \phi'(0) \le \phi'(\pi_{1, ij}^{*}) = \theta^{*}_{1, ij} \le \phi'(\frac{1}{m})&
\end{eqnarray}
holds.
Hence, for every $i$, if we lower-bound $\tau_{i}$, the Newton--Raphson decrement for the $i$th row of $\boldsymbol{\theta^{*}}$ as
\begin{eqnarray}
    \tau_{i} &\leftarrow& \max\{\tau_{i}, \hat{\theta}^{*}_{i} - \phi'\left(\textstyle{\frac{1}{m}} \right)\}, 
\end{eqnarray}
then, for any $j$,
\begin{eqnarray}
    \tilde{\theta}_{ij} - \tau_{i} &\leq& \theta^{*}_{ij} - \tau_{i}\\
                                   &\leq& \hat{\theta}^{*}_{i} - \tau_{i}\\
                                   &\leq& \phi'\left(\textstyle{\frac{1}{m}}\right).
\end{eqnarray}
This means that every element in the $i$th row of $\tilde{\boldsymbol{\theta}}$ computed in line \ref{line6} in Algorithm \ref{ouralgorithm} is no larger than $\phi'(\frac{1}{m})$.  After line~\ref{line7}, $\boldsymbol{\theta}^{*}$ satisfies the condition (\ref{condition_of_pi}). Similarly, we force $\pi_{2, ij}$ to satisfy the following conditions:
\begin{equation}
    \forall i, \ 0\leq \pi_{2, ij} \leq \textstyle{\frac{1}{n}}.
\end{equation}
After line \ref{line11}, this condition is satisfied.

\subsection{Theoretical analysis}
In the presence of outliers, we expect to approximate the OT by preventing mass transport to outliers.
This property is formalized below.
\begin{definition} 
 Suppose $\boldsymbol{\pi} \in \mathbb{R}_{+}^{m \times n}$ and a set of indices O $\subseteq \{ 1, \ldots, n \}$ satisfies the following condition:
\begin{equation}\label{transportnomass}
    \forall i, \pi_{ij} = 0 \ \ \mathrm{if} \ \ j \in O.
\end{equation}
Then, we say $\boldsymbol{\pi}$ transports no mass to O.
\end{definition}
Although we do not expect to transport any mass to outliers, the optimal solution of the CROT must satisfy the coupling constraint and then the condition (\ref{transportnomass}) is never satisfied. To ensure (\ref{transportnomass}), we consider solving the CROT with only a finite number of updates subsequently. Then, an intermediate solution can satisfy (\ref{transportnomass}), although the coupling constraint is not satisfied. This is in stark contrast to the previous works \citep{Chizat2017,Balaji}, which cannot avoid transporting some mass to outliers.\par
The following proposition provides sufficient conditions on the number of iterations $T$ to ensure the condition (\ref{transportnomass}). Refer to Section B in the supplementary file for the proof.
\begin{proposition}\label{proposition}
For a given $z \ (>\frac{\lambda}{\beta - 1})$, let $J\subseteq\{1, \ldots, n \}$ be a subset of indices which satisfies the condition shown in Definition \ref{outlier_definition}.
Suppose we obtained a transport matrix $\boldsymbol{\pi}^{\mathrm{output}}$ by running the alogrithm $T$ times satisfying the following condition:
\begin{equation}\label{robust_condition}
     T < \frac{\frac{z}{\lambda}(\beta - 1) - 1}{(\frac{1}{m})^{\beta - 1} + (\frac{1}{n})^{\beta - 1}}.
\end{equation}
Then, $\boldsymbol{\pi}^{\mathrm{output}}$ transports no mass to J.
\label{theoreticalanalysis}
\end{proposition}
Here, $z>\frac{\lambda}{\beta - 1}$ is necessary so that $T$ is upper-bounded by a positive number.
Intuitively, this means that the transport matrix obtained by Algorithm \ref{ouralgorithm} disregards points distant from inliers more than or equal to $z$.
Note that the condition (\ref{robust_condition}) tells us that a sufficiently small number of iterations $T$ leads to an approximate CROT solution that does not transport any mass to outliers. \par
We discuss the selection of hyperparameters $\beta$ and $\lambda$ in Sections \ref{hyperparameter_selection}.

\section{Experiments}
Here, we show two applications of our method to demonstrate the practical effectiveness. In both of the experiments in Sections~\ref{dataset_distance_section} and \ref{outlierdetection_experiment}, we set the hyperparameters in the proposed method as $\beta = 1.2$ and $\lambda = 2$. We discuss the selection of these hyperparameters in Section~\ref{hyperparameter_selection}.

\subsection{Measuring distance between datasets}\label{dataset_distance_section}
In the first experiments, we numerically confirm that our method can compute the distance more robustly than the Sinkhorn algorithm. We used the following benchmark datasets: MNIST \citep{MNIST}, FashionMNIST \citep{FashionMNIST}, KMNIST \citep{KMNIST}, and EMNIST(Letters)\citep{EMNIST}. From each benchmark dataset, we randomly sampled 10000 data points, and split them into two subsets, each containing 5000 data. We regarded these data points as inliers. Then a portion of one subset was replaced by data from another benchmark dataset which were  regarded as outliers. We computed CROT and outlier-robust CROT between these two subsets, and investigated how they changed when the outlier ratio are 5\%, 10 \%, 15\%, 20\%, 25\% and 30\%.  We simply used the raw data to compute the distance matrix $\gamma_{ij} = \| \boldsymbol{x}_i - \boldsymbol{y}_j \|_{2}^{2}$, i.e., the Euclidean distance between raw data, and used their median value as the threshold $z$. In this way, we expect that outliers are distinguished from inliers.
The results are shown in Figure~\ref{dataset_distance}. Although the distance computed by the Sinkhorn algorithm drastically changes when the outlier ratio gets larger, the degree of change in the output values of our algorithm is milder in every dataset. Therefore, we can see that our algorithm computes the distance between datasets more stably than the Sinkhorn algorithm.

\begin{figure}
    \centering
    \includegraphics[width = 15cm]{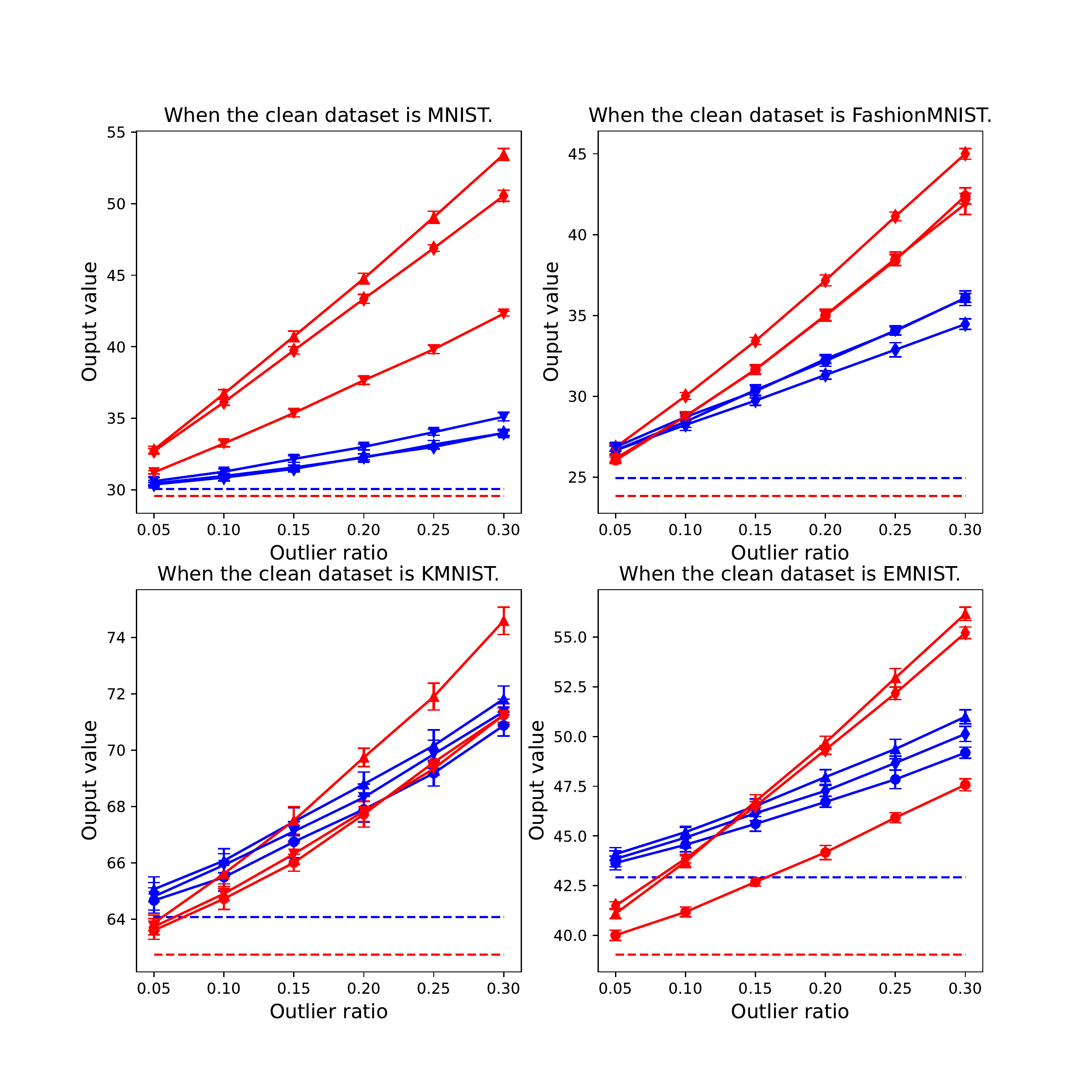}
    \caption{The mean and standard deviation of the output value of the Sinkhorn algorithm (red) and our algorithm (blue) over 20 runs. $\bigcirc$, $\triangle$, $\triangledown$, and $\diamondsuit$ represents when the outlier dataset are MNIST, FashionMNIST, KMNIST, and EMNIST, respectively. The dotted line is the output value when the dataset is clean.}
    \label{dataset_distance}
\end{figure}

\subsection{Applications to outlier detection}\label{outlierdetection_experiment}

\begin{table}
    \centering
  \begin{tabular}{ccc} 
    \toprule
     & Outliers & Inliers \\ 
    \midrule
    One-class SVM & 49.78 $\pm$ 1.83 \% & 49.99 $\pm$ 0.10 \%  \\
    Local outlier factor & 49.43 $\pm$ 3.68 \% & 99.13 $\pm$ 0.11 \% \\
    Isolation forest & 42.78 $\pm$ 6.95 \% & 72.81 $\pm$ 3.34 \% \\
    Elliptical envelope & 95.57 $\pm$ 2.77 \% & 69.37 $\pm$ 5.61 \% \\
    \midrule
    Baseline technique (95th) &  92.78 $\pm$ 1.67  \% & 92.66  $\pm$ 0.44 \% \\
    Baseline technique (97.5th) & 84.19  $\pm$ 2.10 \% & 96.41 $\pm$ 0.32 \% \\
    Baseline technique (99th) & 65.04  $\pm$ 2.75 \% & 98.60 $\pm$ 0.16 \% \\
    \midrule
    ROBOT (95th) &  99.96 $\pm$ 0.08 \% & 68.76  $\pm$ 0.49 \% \\
    ROBOT (97.5th) & 99.89 $\pm$ 0.14 \% & 77.22 $\pm$ 0.63 \% \\
    ROBOT (99th)  & 99.48 $\pm$ 0.31 \% &  84.79 $\pm$ 0.47 \% \\
    \midrule
    Our Method (95th) & 98.98  $\pm$ 0.66 \% & 86.72 $\pm$ 0.72 \% \\
    Our Method (97.5th) & 96.96  $\pm$ 1.71 \% & 91.58 $\pm$ 0.38 \% \\
    Our Method (99th) & 92.25 $\pm$ 1.53 \% & 95.73 $\pm$ 0.34 \% \\
    \bottomrule
  \end{tabular} 
    \caption{The percentage of true outliers/inliers detected as outliers/inliers over 50 runs. The numbers show the mean and standard deviation. ``(Xth)" means $X$th percentile was used in its subsampling phase.}
  \label{outlierdetection}
\end{table}

Our algorithm enables us to detect outliers. Let $\mu_m$ be a clean dataset and $\nu_n$ be a dataset which is polluted with outliers. We regard the $j$th data point in $\nu_n$ is an outlier if Algorithm~\ref{ouralgorithm} outputs a transport matrix whose $j$th column is all zeros.\par 
In this experiment, we used Fashion-MNIST \citep{FashionMNIST} as a clean dataset and MNIST \citep{MNIST} as outliers. $\nu_n$ consists of 9500 images from Fashion-MNIST and 500 images from MNIST. $\mu_m$ consists of 10000 images from Fashion-MNIST. We computed the transport matrix with the two datasets and identified the outlying MNIST images. We simply used the raw data to compute the distance matrix $\gamma_{ij} = \|\boldsymbol{x}_i - \boldsymbol{y}_j\|^2_2$, i.e., the Euclidean distance between raw data. \par

We compared the proposed method with the ``ROBust Optimal Transport" (ROBOT) method \citep{Mukherjee} and the method proposed by \citeauthor{Balaji}(\citeyear{Balaji})
, which are existing methods to compute OT robustly. 
We also compared our method with a variety of popular outlier detection algorithms available in scikit-learn \citep{Scikit-learn}: the one-class support vector machine (SVM) \citep{oneclassSVM}, local outlier factor \citep{localoutlierfactor}, isolation forest \citep{isolationforest}, and elliptical envelope \citep{ellipticalenvelope}. In the ROBOT method, we set the cost truncation hyperparameter to the (1) 95th (2) 97.5th (3) 99th percentile of the distance matrix in the subsampling phase \citep{Mukherjee}. \par
For our method, the distance tolerance parameter $z$ in Definition~\ref{outlier_definition} is necessary to detect outliers by leveraging Proposition~\ref{theoreticalanalysis}. Once $z$ is chosen, after running the algorithm  $\left \lfloor \frac{\frac{z}{\lambda}(\beta - 1) - 1}{(\frac{1}{m})^{\beta - 1} + (\frac{1}{n})^{\beta - 1}} \right \rfloor$ times satisfying the condition (\ref{robust_condition}), points in $\nu_n$ that are far from any points in $\mu_m$ with more than or equal to distance $z$ are regarded as outliers.
To determine $z$, we need a subsampling phase using the clean dataset similar to \citeauthor{Mukherjee}(\citeyear{Mukherjee}). We propose the following heuristics: since we know that $\mu_m$ is clean, we subsample two datasets from it and compute the distance matrix. Then, we choose the minimum value for each row and use the largest value among them as $z$. This procedure is essentially estimating the maximum distance between two samples in the clean dataset. In order to avoid subsampling noise, we used the (1) 95th (2) 97.5th (3) 99th percentile instead of the maximum. Additionaly, we compared our method with a natural baseline to identify a data point as an outlier if the minimum distance to the clean dataset is larger than the distance computed in the subsampling phase. We call this method ``the baseline technique''.
The results are shown in Table \ref{outlierdetection}. One can see that our method has a high performance in detecting not only outliers but also inliers. \par
\begin{table}[t]
    \centering
  \begin{tabular}{ccc}
    \toprule
     & Outliers & Inliers \\ 
    \midrule
    \citeauthor{Balaji}[\citeyear{Balaji}] & 89.0 $\pm$ 16.9  \% & 67.0 $\pm$ 8.9 \% \\
    Our Method & 96.6 $\pm$ 2.0 \% & 88.0 $\pm$ 0.7 \% \\
    \bottomrule
  \end{tabular} 
  \caption{Comparison with \citeauthor{Balaji}[\citeyear{Balaji}] with 1000 data points. The numbers show the mean and standard deviation of the percetage of the true outliers/inliers detected as outliers/inliers over 10 runs.}
  \label{Balaji_table}
\end{table}
We also tried the code of \citeauthor{Balaji}(\citeyear{Balaji}) based on CVXPY \citep{CVXPY}, which is not scalable so that the computational time is not negligible even with 1000 data points. Similar to the previous experiments, the clean dataset $\mu_m$ consists of 1000 Fashion-MNIST data points and the polluted dataset $\nu_n$ consists of 950 Fashion-MNIST data as inliers and 50 MNIST data points as outliers. Table \ref{Balaji_table} shows the mean accuracy and standard deviations over 10 runs. The run-time of their method was $820\pm17$ seconds, while that of our method was $6\pm0.2$ seconds. Our method outperforms the method by \citeauthor{Balaji}[\citeyear{Balaji}] in terms of not only outlier detection performance but also computation time.\par

\subsection{The selection of hyperparamters $\beta$ and $\lambda$}\label{hyperparameter_selection}
Here, we dicuss the selection of hyperparameters $\beta$ and $\lambda$. Figure~\ref{hp_sensitivity_MNIST} shows the sensitivity to the hyperparameters for the same outlier task above. We see $2 \leq \lambda \leq 14$ and $1.2 \leq \beta \leq 1.5$ are good choices of possible hyperparameters. \par
Then, how about other $\beta$ or $\lambda$? Since, we are running the algorithm $\left \lfloor \frac{\frac{z}{\lambda}(\beta - 1) - 1}{(\frac{1}{m})^{\beta - 1} + (\frac{1}{n})^{\beta - 1}} \right \rfloor$ times, if we choose $\beta$ excessively large or $\lambda$ excessively small, we will harmfully increase the computation time. On the other hand, if we choose $\beta$ excessively small or excessively $\lambda$ large, $\left \lfloor \frac{\frac{z}{\lambda}(\beta - 1) - 1}{(\frac{1}{m})^{\beta - 1} + (\frac{1}{n})^{\beta - 1}} \right \rfloor$ will become less than or equal to 0, which means that we can not start running the algorithm.
However, these discussions are when $z$ is fixed. If we scale the raw value of data by constant multiplication, $z$ will also change. By scaling $z$, we can adjust the number of times running the algorithm so that it will be larger than 0, and at the same time, not too large. 
In the MNIST detection task, we scaled the raw data so that the number of times running the algorithm will fit in $(0, 20)$ when $(\beta, \lambda)$ = (1.4, 14), (1.3, 10), (1.3, 12), (1.3, 14), (1.2, 6), (1.2, 8), (1.2, 10), (1.2, 12), (1.2, 14). We can see that scaling the raw data by constant multiplication has no problem in detecting outliers (Figure~\ref{hp_sensitivity_MNIST}). Therefore, since we can scale $z$, we can adjust the number of times running the algorithm with limited $\lambda\in [2, 14]$ and $\beta\in[1.2, 1.5]$.
\par
Below, we confirm that the proposed method is sufficiently stable in the above range of hyperparameters by using the credit card fraud detection dataset\footnote{https://www.kaggle.com/datasets/mlg-ulb/creditcardfraud}. 
We experimentally observe the sensitivity of the proposed method to the choice of the hyperparameters $\beta$ and $\lambda$. We used the credit card fraud detection dataset to verify that the proposed method is sufficiently stable in a certain range of the hyperparameters. \par
The credit card fraud detection dataset contains transactions made by credit cards in 2013 by European cardholders. Due to confidentiality issues, it does not provide the original features and more background information about the data. Instead, it contains 28-dimensional numerical feature vectors, which are the result of a principal component analysis transformation. We used these feature vectors to compute the cost matrix, which is the L2 distance among them. The task is to detect 450 frauds out of 9000 transactions. We conducted ten experiments for each pair of $\beta$ and $\lambda$. \par
We show the results in Figure \ref{hp_sensitivity_CreditCard}. We can see that the detection accuracy is sufficiently stable when $1.2\leq \beta \leq 1.5$ and $1\leq \lambda \leq 14$.

\begin{figure}[t]
    \begin{tabular}{cc}
      \begin{minipage}[t]{0.45\linewidth}
        \centering
        \includegraphics[width = 7.5cm]{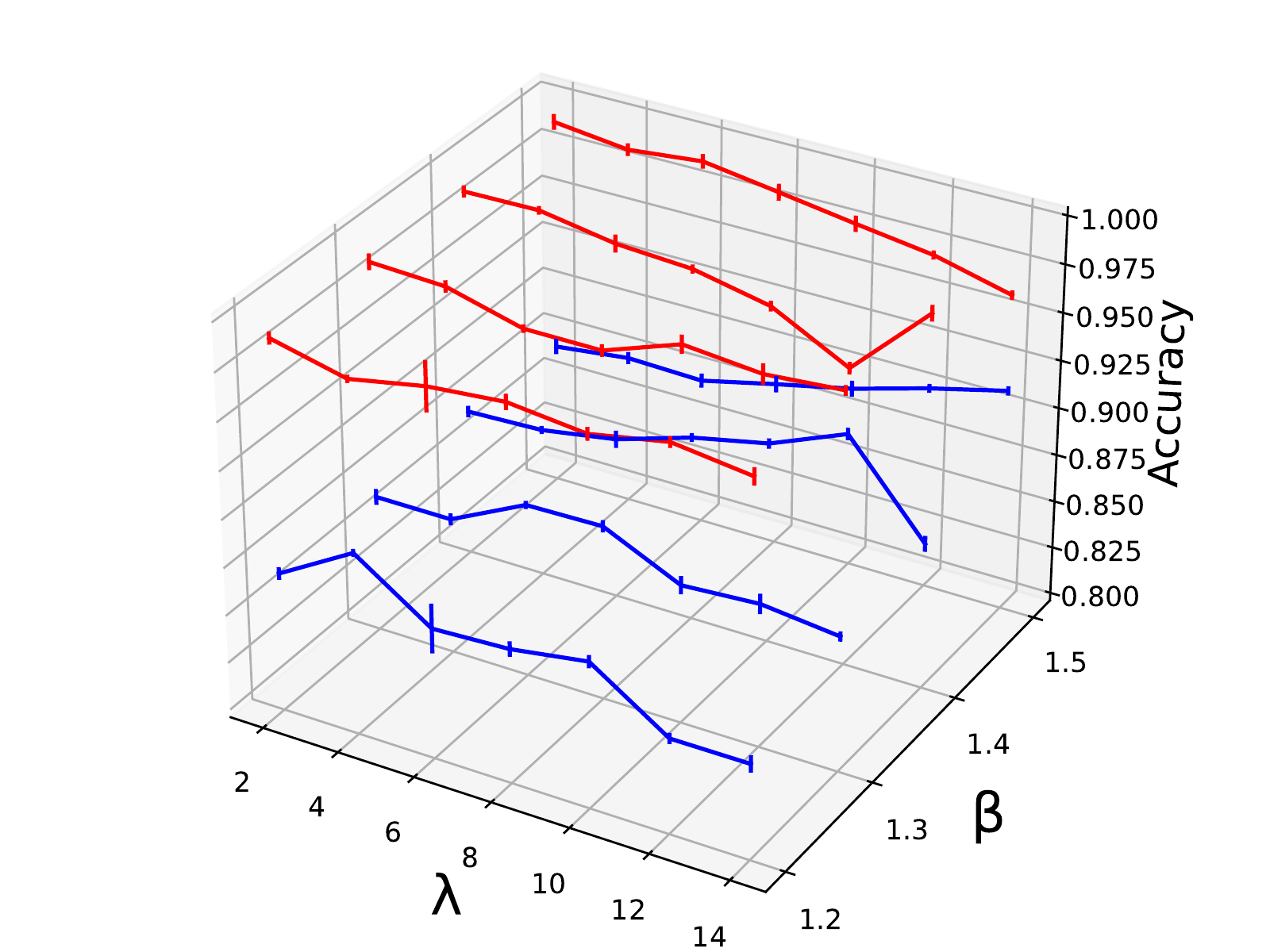}
        \subcaption{The hyperparameter sensitivity in the Fashion-MNIST detection task. (Blue) The inlier detection accuracy. (Red) The outlier detection arruracy. Error bars represent the mean and standard deviation.}
        \label{hp_sensitivity_MNIST}
      \end{minipage} &
      \begin{minipage}[t]{0.45\linewidth}
        \centering
        \includegraphics[width = 7.5cm]{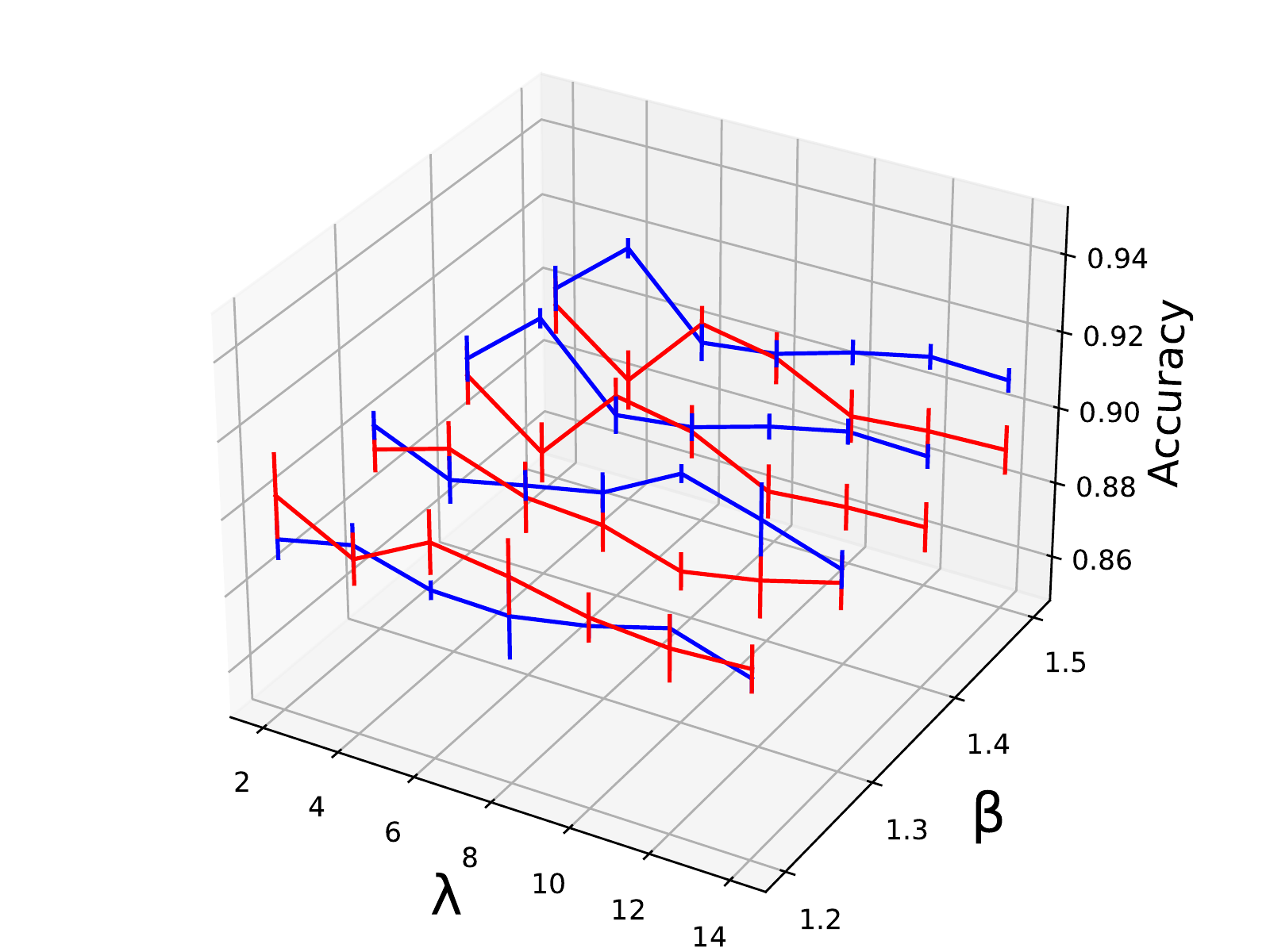}
        \subcaption{The hyperparameter sensitivity in the credit card fraud detection task. (Blue) The inlier detection accuracy. (Red) The outlier detection arruracy. Error bars represent the mean and standard deviation.}
        \label{hp_sensitivity_CreditCard}
      \end{minipage}
    \end{tabular}
      \caption{The hyperparameter ($\beta$ and $\lambda$) sensitivity in the Fashion-MNIST dataset and in the credit card fraud detection dataset.}
 \end{figure}

\section{Conclusion}
In this work, we proposed to robustly approximate OT by regularizing the ordinary OT with the $\beta$-potential term. By leveraging the domain of the Fenchel conjugate of the $\beta$-potential, our algorithm does not move any probability mass to outliers. We demonstrated that our proposed method can be used in estimating a probability distribution robustly even in the presence of outliers and successfully detecting outliers from a contaminated dataset.

\acks{}
SN was supported by JST SPRING, Grant Number JPMJSP2108. MS was supported by JST CREST Grant Number JPMJCR18A2.
\bibliography{acml22}

\newpage
\appendix

\section{Details of the Bregman projection}\label{apd:first}
Here, we demonstrate the details of our algorithm inspired by the Non-negative alternate scaling algorithm (NASA) introduced by \cite{ROTandRMD}, which is an algorithm to obtain the solution for CROT. Although our outlier-robust CROT does not satisfy the assumptions required for CORT, we show our algorithm is constructed similarly to the NASA algorithm. \par
First, we introduce basics of convex analysis as preliminaries. Next, we explain the alternate scaling algorithm, which is the basis of the NASA algorithm, and the required assumptions for it. Finally, we show the NASA algorithm for the separable Bregman divergence, and how we borrowed their idea to construct our algorithm.

\subsection{Convex analysis}
Let $\mathcal{E}$ be a Euclidean space with inner product $\langle \cdot, \cdot \rangle$ and induced norm $\|\cdot\|$. The boundary, interior, and relative interior of a subset $\mathcal{S} \subseteq \mathcal{E}$ are denoted by bd($\mathcal{S}$), int($\mathcal{S}$), and ri($\mathcal{S}$), respectively. Recall that for a convex set $\mathcal{C}$, we have
\begin{equation}
    \mathrm{ri} (\mathcal{C}) = \{ \boldsymbol{x} \in \mathcal{E} \ |\  \forall \boldsymbol{y} \in \mathcal{C}, \  \exists \lambda > 1,\ \lambda\boldsymbol{x} + (1 - \lambda)\boldsymbol{y} \in \mathcal{C} \}.
\end{equation}
In convex analysis, scalar functions are defined over the whole space $\mathcal{E}$ and take values in $\mathbb{R}\cup \{-\infty, \ \infty \}$. The effective domain, or simply domain, of a function $f$ is defined as the set:
\begin{equation}
    \dom f = \{ \boldsymbol{x} \in \mathcal{E}\ | \ f(\boldsymbol{x}) < + \infty \}.
\end{equation}
\begin{definition}[Closed functions]
    A function $f:\mathbb{R}^{n}\rightarrow \mathbb{R}$ is said to be closed if for each $\alpha \in \mathbb{R}$, the sublevel set $\{ \boldsymbol{x} \in \dom f \ | \ f(\boldsymbol{x}) \leq \alpha \}$ is a closed set. 
\end{definition}
If $\dom f$ is closed, then $f$ is closed.
\begin{definition}[Proper functions]
    Suppose a convex function $f:\mathcal{E} \rightarrow \mathbb{R} \cup \{\pm \infty\}$ satisfies $f(\boldsymbol{x}) > -\infty$ for every $\boldsymbol{x} \in \dom f$ and there exists some point $\boldsymbol{x}_0$ in its domain such that $f(\boldsymbol{x}_0) < + \infty$. Then $f$ is called a proper function.
\end{definition}
A proper convex function is closed if and only if it is lower semi-continuous. \footnote{Let $X$ be a topological space. A funtion $f:X\rightarrow\mathbb{R}\cup\{-\infty, \infty\}$ is called lower semi-continuous at a point $x_0 \in X$ if for every $y<f(x_0)$ there exists a neighborhood $U$ of $x_0$ such that $f(x)>y$ for all $x\in U$.} A closed function $f$ is continuous relative to any simplex, polytope of a polyhedral subset in $\dom f$. A convex function $f$ is always continuous in the relative interior ri($\dom f$).
\begin{definition}[Essential smoothness \cite{Heinz}]
Suppose f is a closed convex proper function on $\mathcal{E}$ with $\mathrm{int}(\dom f) \neq \emptyset$. Then f is essentially smooth, if f is differentiable on $\mathrm{int}(\dom f )$ and 
\begin{center}
    $\begin{array}{c}
    \forall n\in\mathbb{N}, \  x_n \in \mathrm{ int}(\dom f),  \\
    x_n \rightarrow x\in \mathrm{bd}(\dom f)
\end{array} \Bigr\} \Rightarrow \|\nabla f(\boldsymbol{x}_n)\| \rightarrow \infty.$
\end{center}

\end{definition}
\begin{definition}[Essential strict convexity \cite{Heinz}]
Let $\partial f$ be the subgradient of $f$. Suppose f is closed convex proper on $\mathcal{E}$. Then, f is essentially strictly convex, if f is strictly convex on every convex subset of $\mathrm{dom(\mathit{\partial f)}}$.
\end{definition}
We define a set of functions called the Legendre type and Fenchel conjugate functions.
\begin{definition}[Legendre type \cite{Heinz}]
Suppose  f is a closed convex proper function on $\mathcal{E}$. Then, f is said to be of the Legendre type if f is both essentially smooth and essentially strictly convex.
\end{definition}
\begin{definition}[Fenchel conjugate \cite{ROTandRMD}]
The Fenchel conjugate $f^*$ of a function $f$ is defined for all $\boldsymbol{y} \in \mathcal{E}$ as follows:
\begin{equation}
    f^{*}(\boldsymbol{y}) = \sup_{\boldsymbol{x}\in \mathrm{int}(\dom f)} \langle \boldsymbol{x}, \boldsymbol{y} \rangle - f(\boldsymbol{x}).
\end{equation}
\end{definition}

The Fenchel conjugate $f^*$ is always a closed convex function and if $f$ is a closed convex function, then $(f^*)^* = f$, and $f$ is of the Legendre type if and only if $f^*$ is of the Legendre type. If $f^*$ is of the Legendre type, the gradient mapping $\nabla f$ is a homeomorphism\footnote{A function $f$ : $X \rightarrow Y$ between two topological spaces is a homeomorphism if it has the following three properties: (a) $f$ is a bijection. (b) $f$ is continuous. (c) The inverse function $f^{-1}$ is continuous.} between int(dom$f$) and int(dom$f^{*}$), with inverse mapping $(\nabla f)^{-1} = \nabla f^*$. This guarantees the existence of dual coordinate systems $\boldsymbol{x}(\boldsymbol{y}) = \nabla f^*(\boldsymbol{y})$ and $\boldsymbol{y}(\boldsymbol{x}) = \nabla f(\boldsymbol{x})$ on int(dom $f$) and int(dom $f^*$). \par
Finally, we say that a function $f$ is a cofinite if it satisfies
\begin{equation}
    \lim_{\lambda \rightarrow + \infty} f(\lambda \boldsymbol{x}) / \lambda = + \infty,
\end{equation}
for all nonzero $\boldsymbol{x} \in \mathcal{E}$. Intuitively, it means that $f$ grows super-linearly in every direction. In particular, a closed convex proper function is cofinite if and only if $\dom f^{*} = \mathcal{E}$.

\subsection{Alternate scaling algorithm}
Here, we show the details of obtaining the Bregman projection onto a convex set. Let $\phi$ be a function of the Legendre type with Fenchel conjugate $\phi^{*} = \psi$. In general, computing Bregman projections onto an arbitrary closed convex set $\mathcal{C}\subseteq\mathcal{E}$ such that $\mathcal{C}\cap \mathrm{int}(\dom\phi) \neq \emptyset$ is nontrivial \cite{ROTandRMD}. Sometimes, it is possible to decompose $\mathcal{C}$ into an intersection of finitely many closed convex sets:
\begin{equation}
    \mathcal{C} = \bigcap_{l = 1}^{s} \mathcal{C}_{l},
\end{equation}
where the individual Bregman projections onto the respective sets $\mathcal{C}_1, \ldots, \mathcal{C}_s$ are easier to compute. It is then possible to obtain the Bregman projections onto $\mathcal{C}$ by alternate projections onto $\mathcal{C}_1, \ldots, \mathcal{C}_s$ according to Dykstra's algorithm \citep{Dykstra}. \par
In more detail, let $\sigma:\mathbb{N}\rightarrow \{ 1, \ldots, s \}$ be a control mapping that determines the sequence of subsets onto which we project. For a given point $\boldsymbol{x}_{0} \in \mathrm{int(dom}\phi)$, the Bregman projection $T_{\mathcal{C}}(\boldsymbol{x}_{0})$ of $\boldsymbol{x}_0$ onto $\mathcal{C}$ can be approximated with Dykstra's algorithm by iterating the following updates:
\begin{equation}\label{DykstraAlgorithm}
    \boldsymbol{x}_{k + 1} \leftarrow T_{\mathcal{C}_{\sigma(k)}}(\nabla \psi (\nabla \phi (\boldsymbol{x}_{k} + \boldsymbol{y}^{\sigma(k)})),
\end{equation}
where the correction term $\boldsymbol{y}^{1}, \ldots, \boldsymbol{y}^{s}$ for the respective subsets are initialized with the null element of $\mathcal{E}$, and are updated after projection as follows:
\begin{equation}\label{DykstraCorrection}
    \boldsymbol{y}^{\sigma(k)} \leftarrow \boldsymbol{y}^{\sigma(k)} + \nabla\phi(\boldsymbol{x}_k) - \nabla \phi (\boldsymbol{x}_{k + 1}).
\end{equation}
Under some technical assumptions, the sequence of updates $(\boldsymbol{x}_k)_{k \in \mathbb{N}}$ converges in terms of some norm to $P_{\mathcal{C}}(\boldsymbol{x_0})$ with a linear rate. Several sets of such conditions have been studied \citep{Dhillon, Tseng, Bauschke}. Here, we use the following conditions proposed by \citeauthor{Dhillon} [\citeyear{Dhillon}] for the CROT framework:
\begin{itemize}
  \item The function $\phi$ is cofinite
  \item The constraint qualification $\mathrm{ri}(\mathcal{C}_1)\cap \cdots \cap \mathrm{ri}(\mathcal{C}_s)\cap\mathrm{int(dom}\phi) \neq \emptyset$ holds
  \item The control mapping $\sigma$ is essentially cyclic, that is , there exists a number $t\in\mathbb{N}$ such that $\sigma$ takes each output value at least once during any $t$ consecutive input values
\end{itemize}
Once these conditions are imposed, the convergence of Dykstra's algorithm is guaranteed.

\subsection{Technical assumptions for CROT to hold}\label{technicalassumptions}
Some mild technical assumptions are required on the convex regularizer $\phi$ and its Fenchel conjugate $\psi = \phi*$ for the CROT framework to hold. The assumptions are as follows:
\begin{enumerate}
    \item $\phi$ is of Legendre type. \label{assumption1}
    \item $(0, 1)^{d \times d} \subseteq \dom\phi$ \label{assumption2}
    \item $\dom\psi = \mathbb{R}^{d \times d}$ \label{assumption3}
\end{enumerate}
Some assumptions relate to required conditions for the definition of Bregman projections and convergence of the algorithms, while others are more specific to CROT problems.\par
The first assumption (\ref{assumption1}) is required for the definition of the Bregman projection. In addition, it guarantees the existence of dual coordinate systems on int($\dom \phi$) and int($\dom \psi$) via the homeomorphism $\nabla\phi = \nabla\psi^{-1}$.\par
The second assumption (\ref{assumption2}) ensures the constraint qualification $\mathcal{G}(\frac{\boldsymbol{1}_m}{m}, \frac{\boldsymbol{1}_n}{n}) \cap \mathrm{int}(\dom \phi)$ for the Bregman projection onto the transport polytope. \par
The third assumption (\ref{assumption3}) equivalently requires $\phi$ to be cofinite for convergence. 
\subsection{NASA algorithm}
In this subsection, we show the NASA algorithm based on Dykstra's algorithm constructed by projections on $\mathcal{C}_0$,$ \mathcal{C}_1$, and $\mathcal{C}_2$.
\begin{algorithm}[t]
\renewcommand{\thealgorithm}{}
\caption{NASA algorithm}
\label{NASA_separable}
\begin{algorithmic}
\STATE $\tilde{\boldsymbol{\theta}} \leftarrow - \boldsymbol{\gamma} / \lambda$
\STATE $\boldsymbol{\theta^{*}} \leftarrow  \max \{\nabla \phi(\boldsymbol{0}_{m \times n}), \tilde{\boldsymbol{\theta}} \}$
\REPEAT 
  \STATE $\boldsymbol{\tau} \leftarrow \boldsymbol{0}_{m}$
  \REPEAT
    \STATE $\boldsymbol{\tau} \leftarrow \boldsymbol{\tau} + \frac{\nabla \psi(\boldsymbol{\theta}^{*} - \boldsymbol{\tau}\boldsymbol{1}_{n}^{\top})\boldsymbol{1}_{n} - \frac{\boldsymbol{1}_{m}}{m}}{ \nabla^2 \psi(\boldsymbol{\theta}^{*}- \boldsymbol{\tau}\boldsymbol{1}_{n}^{\top}) \boldsymbol{1}_{n}}$
    \UNTIL convergence
  \STATE $\boldsymbol{\tilde{\theta}} \leftarrow \boldsymbol{\tilde{\theta}} - \boldsymbol{\tau} \boldsymbol{1}_{n}^{\top}$
  \STATE $\boldsymbol{\theta^{*}} \leftarrow \max \{ \nabla \phi(\boldsymbol{0}_{m \times n}), \boldsymbol{\tilde{\theta}\}}$
  \STATE $\boldsymbol{\sigma} \leftarrow \boldsymbol{0}_{n}^{\top}$
  \REPEAT
    \STATE $\boldsymbol{\sigma} \leftarrow \boldsymbol{\sigma} + \frac{\boldsymbol{1}_{m}^{\top} \nabla \psi(\boldsymbol{\theta}^{*} - \boldsymbol{1}_{m}\boldsymbol{\sigma}) - (\frac{\boldsymbol{1}_{n}}{n})^{\top}}{\boldsymbol{1}_{m}^{\top} \nabla^2\psi(\boldsymbol{\theta}^{*} - \boldsymbol{1}_{m}\boldsymbol{\sigma})}$
    \UNTIL convergence
  \STATE $\boldsymbol{\tilde{\theta}} \leftarrow \boldsymbol{\tilde{\theta}} - \boldsymbol{1}_{m} \boldsymbol{\sigma} $
  \STATE $\boldsymbol{\theta^{*}} \leftarrow \max \{\nabla \phi(\boldsymbol{0}_{m \times n}), \boldsymbol{\tilde{\theta}} \}$

\UNTIL convergence
\STATE $\boldsymbol{\pi^{*} \leftarrow \nabla\psi(\theta^{*})}$
\end{algorithmic}
\end{algorithm}
\subsubsection{Projcetion onto $\mathcal{C}_0$}
Let us consider the projection of given matrix $\overline{\boldsymbol{\pi}}$ onto $\mathcal{C}_0$. We denote this projection $P_{\mathcal{C}_0}(\boldsymbol{\overline{\pi}})$ by $\boldsymbol{\pi}_{0}^{*}$. Then, the Karush-Kuhn-Tucker conditions \citep{KKT_1, KKT_2} for $\boldsymbol{\pi}_{0}^{*}$ are as follows:
\begin{eqnarray}
    \boldsymbol{\pi}_{0}^{*} & \geq & \mathbf{0}, \label{primalfeasibility} \\
    \nabla \phi(\boldsymbol{\pi}_0^{*}) - \nabla \phi(\overline{\boldsymbol{\pi}}) &\geq& \mathbf{0} \label{dualfeasibility}, \\
    (\nabla \phi (\boldsymbol{\pi}_{0}^{*}) - \nabla \phi (\overline{\boldsymbol{\pi}})) \odot \boldsymbol{\pi}_{0}^{*} &=& \mathbf{0}, \label{complementaryslackness} \
\end{eqnarray}
where (\ref{primalfeasibility}) is the primal feasibility, (\ref{dualfeasibility}) is the dual feasibility, and (\ref{complementaryslackness}) is the complementary slackness. \par
Since we are thinking of the separable Bregman divergence, the projection onto $\mathcal{C}_{0}$ can be performed with a closed-form expression on primal parameters:
\begin{equation}\label{non-negativity_constraint}
    \pi^{*}_{0, ij} = \max \{ 0, \overline{\pi}_{ij} \},
\end{equation}
where, $\pi^{*}_{0, ij}$ is the $(i, j)$-element of matrix $\boldsymbol{\pi}_{0}^{*}$. 
Since $\phi'$ is increasing, this is equivalent on the dual parameters of $\boldsymbol{\pi}_0^*$, $\boldsymbol{\theta}_0^*$, to
\begin{equation} \label{projectionC0}
    \theta_{0, ij}^{*} = \max \{ \phi'(0), \overline{\theta}_{ij} \}.
\end{equation}
Here, the dual coordinate of the input matrix $\overline{\boldsymbol{\pi}}$ is denoted by $\overline{\boldsymbol{\theta}}$.

\subsubsection{Projection onto $\mathcal{C}_1$ and $\mathcal{C}_1$}
Next, we consider the Bregman projections of a given matrix $\overline{\boldsymbol{\pi}} \in \mathrm{int(dom}\phi)$ onto $\mathcal{C}_{1}$ and $\mathcal{C}_{2}$. 
For the projection onto $\mathcal{C}_1$ and $\mathcal{C}_2$, we employ the method of Lagrange multipliers. The Lagrangians with Lagrange multipliers $\boldsymbol{\mu} \in \mathbb{R}^{m}$ and $\boldsymbol{\nu} \in \mathbb{R}^{n}$ for the Bregman projections $\boldsymbol{\pi}_{1}^{*}$  and $\boldsymbol{\pi}_{2}^{*}$ of a given matrix $\overline{\boldsymbol{\pi}} \in \mathrm{int(dom}\phi)$ onto $\mathcal{C}_1$ and $\mathcal{C}_2$ respectively write as follows:
\begin{eqnarray}
    \mathcal{L}_1 (\boldsymbol{\pi}, \boldsymbol{\mu}) & = & \phi(\boldsymbol{\pi}) - \langle \boldsymbol{\pi}, \nabla \phi(\overline{\boldsymbol{\pi}}) \rangle + \mu^{\top}(\boldsymbol{\pi} \mathbf{1} - \frac{\mathbf{1}}{m}),\\
    \mathcal{L}_2 (\boldsymbol{\pi}, \boldsymbol{\nu}) & = & \phi(\boldsymbol{\pi}) - \langle \boldsymbol{\pi}, \nabla \phi(\overline{\boldsymbol{\pi}}) \rangle + \nu^{\top}(\boldsymbol{\pi}^{\top} \mathbf{1} - \frac{\mathbf{1}}{n}).
\end{eqnarray}
Their gradients are given on int(dom$\phi$) by
\begin{eqnarray}
    \nabla \mathcal{L}_{1}(\boldsymbol{\pi}, \boldsymbol{\mu}) & = & \nabla \phi (\boldsymbol{\pi}) - \nabla \phi(\overline{\boldsymbol{\pi}}) + \boldsymbol{\mu}\mathbf{1}^{\top}, \\
    \nabla \mathcal{L}_{2}(\boldsymbol{\pi}, \boldsymbol{\nu}) & = & \nabla \phi (\boldsymbol{\pi}) - \nabla \phi(\overline{\boldsymbol{\pi}}) + \mathbf{1}\boldsymbol{\nu}^{\top},
\end{eqnarray}
and vanish at $\boldsymbol{\pi}_{1}^{*}, \boldsymbol{\pi}^{*}_{2} \in \mathrm{int(dom}\phi)$ if and only if
\begin{eqnarray}
    \boldsymbol{\pi}_{1}^{*} & = & \nabla \psi (\nabla \phi(\overline{\boldsymbol{\pi}}) - \boldsymbol{\mu}\mathbf{1}^{\top}), \\
    \boldsymbol{\pi}_{2}^{*} & = & \nabla \psi (\nabla \phi(\overline{\boldsymbol{\pi}}) - \mathbf{1} \boldsymbol{\nu}^{\top}).
\end{eqnarray}
By duality, the Bregman projections onto $\mathcal{C}_1$, $\mathcal{C}_2$ are thus equivalent to finding the unique vectors $\boldsymbol{\mu}$, $\boldsymbol{\nu}$, such that the rows of $\boldsymbol{\pi}_{1}^{*}$ sum up to $\frac{\mathbf{1}}{m}$, respectively the columns of $\boldsymbol{\pi}_{2}^{*}$ sum up to $\frac{\mathbf{1}}{n}$:
\begin{eqnarray}
    \nabla \psi (\nabla \phi(\overline{\boldsymbol{\pi}}) - \boldsymbol{\mu}\mathbf{1}^{\top}) \mathbf{1} & = & \frac{\mathbf{1}}{m}, \\ 
    \nabla \psi (\nabla \phi(\overline{\boldsymbol{\pi}}) - \mathbf{1}\boldsymbol{\nu}^{\top})^{\top} \mathbf{1} & = & \frac{\mathbf{1}}{n}.
\end{eqnarray}
Again, since we are resticting ourselves to the separable Bregman divergence, we can compute the projection step more efficiently. Due to the separability, the projections onto $\mathcal{C}_{1}$ and $\mathcal{C}_{2}$ can be divided into $m$ and $n$ parallel subproblems in the search space of 1-dimension as follows:
\begin{eqnarray}
    \sum_{j = 1}^{n} \psi'(\overline{\theta}_{ij} - \mu_{i}) & = & \frac{1}{m}, \\
    \sum_{i = 1}^{m} \psi'(\overline{\theta}_{ij} - \nu_{j}) & = & \frac{1}{n}.
\end{eqnarray}
Here, we denote the dual coordinate of $\overline{\boldsymbol{\pi}}$ by $\overline{\boldsymbol{\theta}}$.\par
In order to obtain the Lagrange multipliers $\mu_{i}$ and $\nu_{j}$, we use the Newton-Raphson method. More specifically, we eploit the following functions:
\begin{eqnarray}
    f(\mu_{i}) &=& - \sum_{j = 1}^{n} \psi' (\overline{\theta}_{ij} - \mu_{i}), \\
    g(\nu_{j}) &=& - \sum_{i = 1}^{m} \psi' (\overline{\theta}_{ij} - \nu_{j}).
\end{eqnarray}
These functions are defined on the open intervals $(\hat{\theta}_{i} - \theta_{\mathrm{limit}}, + \infty)$ and $(\check{\theta}_{j} - \theta_{\mathrm{limit}}, + \infty)$, where $0<\theta_{\mathrm{limit}}<+\infty$ is such that $\dom\psi = (-\infty, \theta_{\mathrm{limit}})$, and $\hat{\theta}_{i} = \max\{ \overline{\theta}_{ij} \}_{1\leq j \leq n}$, $\check{\theta}_{j} = \max\{ \overline{\theta}_{ij} \}_{1\leq i \leq m}$. We can now obtain the unique solution to $f(\mu_{i}) = - \frac{1}{m}$ and $g(\nu_{j}) = - \frac{1}{n}$.
Starting with $\mu_{i} = 0$ and $\nu_{j} = 0$, the Newton-Raphson updates:
\begin{eqnarray}
    \mu_{i} &\leftarrow& \mu_{i} + \frac{\sum_{j = 1}^{n} \psi'(\overline{\theta}_{ij} - \mu_{i}) - \frac{1}{m}}{\sum_{j = 1}^{n} \psi'' (\overline{\theta}_{ij} - \mu_{i})}, \\
    \nu_{j} &\leftarrow& \nu_{j} + \frac{\sum_{i = 1}^{m} \psi'(\overline{\theta}_{ij} - \nu_{i}) - \frac{1}{n}}{\sum_{i = 1}^{m} \psi'' (\overline{\theta}_{ij} - \nu_{i})},
\end{eqnarray}
converge to the optimal solution with a quadraitic rate. To avoid storing the intermediate Lagrange multipliers, the updates can be directly written in terms of the dual parameters:
\begin{eqnarray}
    \theta^{*}_{1, ij} &\leftarrow& \theta^{*}_{1, ij} - \frac{\sum_{j = 1}^{n} \psi'(\theta^{*}_{1, ij}) - \frac{1}{m}}{\sum_{j = 1}^{n} \psi'' (\theta^{*}_{1, ij})}, \\
    \theta^{*}_{2, ij} &\leftarrow& \theta^{*}_{2, ij} - \frac{\sum_{i = 1}^{m} \psi'(\theta^{*}_{2, ij}) - \frac{1}{n}}{\sum_{i = 1}^{m} \psi'' (\theta^{*}_{2, ij})},
\end{eqnarray}
after initilalization by $\theta^{*}_{1, ij} \leftarrow \overline{\theta}_{ij}$, $\theta^{*}_{2, ij} \leftarrow \overline{\theta}_{ij}$. Here, $\theta_{1, ij}^{*}$ and $\theta_{2, ij}^{*}$ are the $i$th row and $j$th column of $\boldsymbol{\theta}^{*}_{1}$ and $\boldsymbol{\theta}^{*}_{2}$ respectively. $\boldsymbol{\theta}^{*}_{1}$ and $\boldsymbol{\theta}^{*}_{2}$ are the dual coordinates of $\boldsymbol{\pi}^{*}_{1}$ and $\boldsymbol{\pi}^{*}_{2}$ respectively.\par
From the above, starting from $\boldsymbol{\xi}$ and writing the successive vectors $\boldsymbol{\mu}^{(k)}$, $\boldsymbol{\nu}^{(k)}$ along iterations, we have:
\begin{eqnarray}
    \psi'(- \boldsymbol{\gamma}/\lambda) &\rightarrow& \psi'\Bigl(\max\{ \phi'(\mathbf{0}), \ -\boldsymbol{\gamma}/\lambda \} \Bigr) \nonumber\\
                                         &\rightarrow& \psi'\Bigl(\max\{ \phi'(\mathbf{0}), \ -\boldsymbol{\gamma}/\lambda \} - \boldsymbol{\mu}^{(1)}\mathbf{1}^{\top}\Bigr) \nonumber\\
                                         &\rightarrow& \psi'\Bigl(\max\{ \phi'(\mathbf{0}), \ -\boldsymbol{\gamma}/\lambda - \boldsymbol{\mu}^{(1)}\mathbf{1}^{\top}\ ) \} \Bigr) \nonumber\\
                                         &\rightarrow& \psi'\Bigl(\max\{ \phi'(\mathbf{0}), - \boldsymbol{\gamma}/\lambda - \boldsymbol{\mu}^{(1)}\mathbf{1}^{\top})\} - \mathbf{1}\boldsymbol{\nu}^{(1)\top} \Bigr) \nonumber\\
                                         &\rightarrow& \psi'\Bigl(\max\{ \phi'(\mathbf{0}), - \boldsymbol{\gamma}/\lambda - \boldsymbol{\mu}^{(1)}\mathbf{1}^{\top} - \mathbf{1}\boldsymbol{\nu}^{(1)\top} \} \Bigr) \nonumber\\
                                         &\rightarrow& \psi'\Bigl(\max\{ \phi'(\mathbf{0}), - \boldsymbol{\gamma}/\lambda - \boldsymbol{\mu}^{(1)}\mathbf{1}^{\top}) - \mathbf{1}\boldsymbol{\nu}^{(1)\top} \} + \boldsymbol{\mu}^{(1)}\mathbf{1}^{\top} - \boldsymbol{\mu}^{(2)}\mathbf{1}^{\top} \Bigr) \nonumber\\
                                         &\rightarrow& \psi'\Bigl(\max\{ \phi'(\mathbf{0}), - \boldsymbol{\gamma}/\lambda - \boldsymbol{\mu}^{(2)}\mathbf{1}^{\top} - \mathbf{1}\boldsymbol{\nu}^{(1)\top} \} \Bigr) \nonumber\\
                                         &\rightarrow& \psi'\Bigl(\max\{ \phi'(\mathbf{0}), - \boldsymbol{\gamma}/\lambda - \boldsymbol{\mu}^{(2)}\mathbf{1}^{\top}) - \mathbf{1}\boldsymbol{\nu}^{(1)\top} \} + \mathbf{1}\boldsymbol{\nu}^{(1)\top} - \mathbf{1}\boldsymbol{\nu}^{(2)\top} \Bigr) \nonumber\\
                                         &\rightarrow& \psi'\Bigl(\max\{ \phi'(\mathbf{0}), - \boldsymbol{\gamma}/\lambda - \boldsymbol{\mu}^{(2)}\mathbf{1}^{\top} - \mathbf{1}\boldsymbol{\nu}^{(2)\top} \} \Bigr) \nonumber\\
                                         &\rightarrow& \cdots \nonumber\\
                                         &\rightarrow& \psi'\Bigl(\max\{ \phi'(\mathbf{0}), - \boldsymbol{\gamma}/\lambda - \boldsymbol{\mu}^{(k)}\mathbf{1}^{\top} - \mathbf{1}\boldsymbol{\nu}^{(k)\top} \} \Bigr) \nonumber\\
                                         &\rightarrow& \psi'\Bigl(\max\{ \phi'(\mathbf{0}), - \boldsymbol{\gamma}/\lambda - \boldsymbol{\mu}^{(k)}\mathbf{1}^{\top} - \mathbf{1}\boldsymbol{\nu}^{(k)\top} \} + \boldsymbol{\mu}^{(k)} \mathbf{1}^{\top} - \boldsymbol{\mu}^{(k + 1)}\mathbf{1}^{\top} \Bigr) \nonumber\\
                                         &\rightarrow& \psi'\Bigl(\max\{ \phi'(\mathbf{0}), - \boldsymbol{\gamma}/\lambda - \boldsymbol{\mu}^{(k+1)}\mathbf{1}^{\top} - \mathbf{1}\boldsymbol{\nu}^{(k)\top} \} \Bigr) \nonumber\\
                                         &\rightarrow& \psi'\Bigl(\max\{ \phi'(\mathbf{0}), - \boldsymbol{\gamma}/\lambda - \boldsymbol{\mu}^{(k+1)}\mathbf{1}^{\top} - \mathbf{1}\boldsymbol{\nu}^{(k)\top} \} + \mathbf{1}\boldsymbol{\nu}^{(k)\top} - \mathbf{1}\boldsymbol{\nu}^{(k + 1)\top} \Bigr) \nonumber\\
                                         &\rightarrow& \psi'\Bigl(\max\{ \phi'(\mathbf{0}), - \boldsymbol{\gamma}/\lambda - \boldsymbol{\mu}^{(k+1)}\mathbf{1}^{\top} - \mathbf{1}\boldsymbol{\nu}^{(k + 1)\top} \} \Bigr) \nonumber \\
                                         &\rightarrow& \cdots \nonumber \\
                                         &\rightarrow& \boldsymbol{\pi}^{*}. \nonumber
\end{eqnarray}

An efficent algorithm then exploits the differences $\boldsymbol{\tau}^{(k)} = \boldsymbol{\mu}^{(k)} - \boldsymbol{\mu}^{(k - 1)}$ and $\boldsymbol{\sigma}^{(k)} = \boldsymbol{\nu}^{(k)} - \boldsymbol{\nu}^{(k - 1)}$ to scale the rows and columns (Algorithm \ref{NASA_separable}).

\subsection{The different point of our algorithm from NASA}
\begin{table}
    \centering
    \caption{Domain of Euclidean norm and $\beta$-potential ($\beta > 1$).}
  \begin{tabular}{ccc} \hline
    Regularization term & dom $\phi$ & dom $\psi$  \\ \hline
    $\beta$-potential ( $\beta > 1$ ) & $\mathbb{R}_{+}$ &  $(\frac{1}{1 - \beta}, \infty)$ \\ \hline
    Euclidean norm & $\mathbb{R}$ & $\mathbb{R}$ \\ \hline
  \end{tabular} 
  \label{domain_of_regularizer}
\end{table}

An example of applying NASA algorithm is when the regulariler is the Euclidean norm $\phi(\pi) = \frac{1}{2}(\pi - 1)^2$. As it is shown in Table \ref{domain_of_regularizer}, we can easily confirm the Euclidean norm satisfies the three assumptions introduced in \ref{technicalassumptions}. 
However, for the outlier-robust CROT, we use $\beta$-potential ($\beta > 1$) as the regularizer, which violates the third assumption, $\dom \psi = \mathbb{R}$ (Table \ref{domain_of_regularizer}). Therefore, we cannot naively apply the NASA algorithm for the outlier-robust CROT. For instance, lines 2, 7, and 11 in our algorithm are not mathematically correct as projections onto $\mathcal{C}_0$. Similarly, lines 4--6 and 8--10 are not mathematically correct for projections onto $\mathcal{C}_1$ and $\mathcal{C}_2$, respectively. \par
In spite of these mathematical issues, we still see lines 2, 7, and 11 in our algorithm as projections onto $\mathcal{C}_0$. In addition, since we cannot update the Newton-Raphson more than twice for projections onto $\mathcal{C}_1$ and $\mathcal{C}_2$ because $\boldsymbol{\theta}^{*} \in \dom \nabla \psi (=\dom \nabla \psi)$ is no longer guaranteed, we overcome this issue by only updating it once.

\section{The proof of Proposition 1}
\begin{proposition}
For a given $z \ (>\frac{\lambda}{\beta - 1})$, let $J\subseteq\{1, \ldots, n \}$ be a subset of indices which satisfies the condition shown in Definition 2.
Suppose we obtained a transport matrix $\boldsymbol{\pi}^{\mathrm{output}}$ by running the alogrithm $T$ times satisfying the following condition:
\begin{equation}\label{robust_condition}
     T < \frac{\frac{z}{\lambda}(\beta - 1) - 1}{(\frac{1}{m})^{\beta - 1} + (\frac{1}{n})^{\beta - 1}}.
\end{equation}
Then, $\boldsymbol{\pi}^{\mathrm{output}}$ transports no mass to J.
\label{theoreticalanalysis}
\end{proposition}
\begin{proof}
 Before the algorithm starts, 
\begin{equation}
    - \frac{z}{\lambda} < \frac{1}{1 - \beta}
\end{equation}
holds. Since every element in $\boldsymbol{\theta}^{*}$ is greater than or equal to $\phi'(0) = \frac{1}{1 - \beta}$, the following inequality holds for every $i$ in the algorithm:
\begin{eqnarray}
    \tau_i & \geq & \frac{1}{1 - \beta} - \phi'\Bigl(\frac{1}{m} \Bigr)\\
           & = &\frac{1}{1 - \beta} - \left (\frac{1}{\beta - 1} \left ( \left (\frac{1}{m} \right)^{\beta - 1} - 1 \right ) \right ) \nonumber \\
    & = & - \frac{1}{\beta - 1} \left ( \frac{1}{m} \right )^{\beta - 1}.
\end{eqnarray}
Therefore, 
\begin{eqnarray}
    -\tau_i &\leq & \frac{1}{\beta - 1} \left ( \frac{1}{m} \right )^{\beta - 1}.
\end{eqnarray}
Similarly, for every $j$, the following inequality holds:
\begin{equation}
    - \sigma_j \leq \frac{1}{\beta - 1} \left ( \frac{1}{n} \right )^{\beta - 1}.
\end{equation}
Therefore, if the algorithm finished running $T$ times and the following inequality holds,
\begin{equation}
    - \frac{z}{\lambda} + T\times \frac{1}{\beta - 1} \left ( \frac{1}{m} \right )^{\beta - 1} + T \times \frac{1}{\beta - 1} \left ( \frac{1}{n} \right )^{\beta - 1}  < \frac{1}{1 - \beta},
\end{equation}
then,
\begin{eqnarray}
    \forall i, \tilde{\theta}_{ij} < \frac{1}{1 - \beta} \ \ \mathrm{if} \ \ j \in \textit{J} \\
\end{eqnarray}
holds.
Therefore,
\begin{eqnarray}
    \forall i, \boldsymbol{\pi}^{\mathrm{output}}_{ij} = 0 \ \ \mathrm{if} \ \ j \in \textit{J}.
\end{eqnarray}
\end{proof}

\end{document}